\documentclass[10pt, conference, letterpaper]{IEEEtran}
\IEEEoverridecommandlockouts
\pdfoutput=1
\usepackage{cite}
\usepackage{amsmath,amssymb,amsfonts}
\usepackage{algorithmic}
\usepackage{graphicx}
\usepackage{textcomp}
\usepackage{comment}
\usepackage{autobreak}
\usepackage{xcolor}
\usepackage{array}
\usepackage{amsmath,amssymb,amsfonts}
\usepackage{amsthm}
\usepackage{mathrsfs}
\usepackage{CJK}
\usepackage[ruled,linesnumbered]{algorithm2e}
\usepackage{amsmath}
\usepackage{url}
\usepackage{color}
\usepackage{pifont}
\usepackage{braket}
\usepackage{multirow}

\newtheorem{assumption}{Assumption}
\newtheorem{lemma}{Lemma}
\newtheorem{theorem}{Theorem}

\newcommand{\tabincell}[2]{\begin{tabular}{@{}#1@{}}#2\end{tabular}}

\allowdisplaybreaks[4]

\ifCLASSINFOpdf
\else
\fi
\begin{document}
\title{Optimal Rate Adaption in Federated Learning with Compressed Communications}

\author{\IEEEauthorblockN{Laizhong Cui\IEEEauthorrefmark{1}, Xiaoxin Su\IEEEauthorrefmark{1}, Yipeng Zhou\IEEEauthorrefmark{2} \IEEEauthorrefmark{3} and Jiangchuan Liu\IEEEauthorrefmark{4}}
\IEEEauthorblockA{\IEEEauthorrefmark{1}College of Computer Science and Software Engineering, Shenzhen University, Shenzhen, China\\
Email: cuilz@szu.edu.cn, suxiaoxin2016@163.com}
\IEEEauthorblockA{\IEEEauthorrefmark{2}School  of Computing, Faculty of Science and Engineering, Macquarie University, Sydney, Australia\\Email: yipeng.zhou@mq.edu.au\\}
\IEEEauthorblockA{\IEEEauthorrefmark{3}
Peng Cheng Laboratory, Shenzhen, China\\}
\IEEEauthorblockA{\IEEEauthorrefmark{4}School of Computing Science, Simon Fraser University, Canada\\
Email: jcliu@cs.sfu.ca}

\IEEEauthorblockA{
	\thanks{
		\newline This work has been partially supported by National Key R\&D Program of China under Grant No.2018YFB1800302 and No.2018YFB1800805, National Natural Science Foundation of China under Grant No.61772345, Shenzhen Science and Technology Program under Grant No. RCYX20200714114645048, No. JCYJ20190808142207420 and No. GJHZ20190822095416463, and Pearl River Young Scholars funding of Shenzhen University, and Australia Research Council DE180100950, and the Major Key Project of PCL (PCL2021A08).}
		\thanks{\textit{(Corresponding author: Yipeng Zhou)}
	}
}
}

\maketitle

\begin{abstract}

Federated Learning (FL) incurs high communication overhead, which can be greatly alleviated by compression for model updates. Yet the tradeoff between compression and model accuracy in the networked environment remains unclear and, for simplicity, most implementations adopt a fixed compression rate only. In this paper, we for the first time systematically examine this tradeoff, identifying  the influence of the compression error on the final model accuracy with respect to the learning rate. Specifically, we factor the compression error of each global iteration into the convergence rate analysis under both strongly convex and non-convex loss functions. We then present an adaptation framework to maximize the final model accuracy by strategically adjusting the compression rate in each iteration. We have discussed the key implementation issues of our framework in practical networks with representative compression algorithms. Experiments over the popular MNIST and CIFAR-10 datasets confirm that our solution effectively reduces network traffic yet maintains high model accuracy in FL.

%To preserve data privacy, federated learning (FL) has emerged that enables decentralized clients to collaboratively train machine learning models over %the Internet. Due to the fact that transmitting model updates in FL gives rise to heavy communication overhead, compressing algorithms have been %widely adopted to compress model updates into a number of centroid values. However, existing works commonly set uniform compression rates during %entire training processes. In this work, \emph{we come up with the framework that can maximize the final model accuracy of FL by adapting  the %compression rate in accordance with  the learning rate in each global iteration.} 
%Specifically, we factor the compression error of each global iteration into the convergence rate analysis, which shows that the influence of %compression errors on the final model accuracy is related with learning rates. Based on derived convergence rates, we build the framework to maximize %the final model accuracy by adapting compression rates, which is applicable for various unbiased compression algorithms. A case study with popular PQ %and QSGD compression algorithms is conducted to illustrate how our framework works. Extensive experiments with MNIST and CIFAR-10 datasets are %conducted, which confirm that adapting compression rates can effectively improve the model accuracy in FL. 
\end{abstract}



% \fi
\IEEEpeerreviewmaketitle

\section{Introduction}
In today's networked world, data are generated and stored everywhere \cite{7488250}. While conventional learning tools are mostly centralized, relying on cloud datacenters to aggregate and analyze the data, {\em Federated Learning} (FL) has been built distributed \cite{mcmahan2017communication}, allowing clients physically remote from each other to collaboratively train learning models over the Internet. It effectively utilizes the rich data spread across different geo-locations and organizations, without compromising their privacy ~\cite{8859260},  \emph{e.g.}, 
for smart devices to predict human trajectories \cite{feng2020pmf} or for healthcare systems to develop predictive models \cite{brisimi2018federated}.

To coordinate the participating clients, a {\em parameter server} (PS) has to be deployed in a FL system for collecting, aggregating and distributing model updates, often in many iterations. This inevitably incurs huge communication overhead, which severely slows down the training process or even makes FL impractical if the model is of high dimensions \cite{huang2020physical}. To address this challenge, compression algorithms have been employed by FL clients \cite{konevcny2016federated}, which use {\em  quantization} or {\em sparsification} to reduce the size of a model update. 
For instance, TernGrad \cite{wen2017terngrad}  quantifies each model update to one of three centroid values in an unbiased manner. This can speed up the training of AlexNet on 8 GPUs by 3.4 times. Assuming that each original model update takes 4 bytes; indexing model updates into four centroids will only need 2 bits, or only $6.25\%$ of the original traffic for each update from the client to the PS \cite{konevcny2016federated}.  

While model compression accelerates computing and communication, it potentially reduces the final model accuracy \cite{sattler2019robust}. For simplicity, most of the existing compression algorithms have fixed their compression rates \cite{konevcny2016federated}, which can limit their applicability and effectiveness. There have been significant studies on the convergence analysis of FL \cite{li2019convergence, yang2021achieving}. The impact of model compression however has yet to be explored, not to mention the optimal rate configuration with compression.

In this paper, we for the first time systematically examine the tradeoff between the compression rate and model accuracy for FL in the networked environment. We show that the influence of the compression error on the final model accuracy is closely related to the learning rate. Specifically, we factor the compression error of each global iteration into the convergence rate analysis given both strongly convex and non-convex loss functions. 
We then present an adaptation framework to maximize the final model accuracy by strategically adjusting the compression rate in each global iteration. Our solution works well for dynamic networks and is generally applicable to different unbiased compression algorithms. We have discussed the key implementation issues of our framework in practical networks, with case studies of two representative compression algorithms, namely,  PQ \cite{suresh2017distributed} and QSGD \cite{alistarh2017qsgd}. 
%In other words, by substituting the compression errors of PQ and QSGD into our theoretical framework, we can optimally determine the number of centroids in each global iteration so as to maximize the final model accuracy.  
Experiments over the popular MNIST and CIFAR-10 datasets \cite{krizhevsky2009learning} confirm that our framework can effectively improve the final model accuracy of FL under the same network bandwidth constraints for both IID and non-IID sample distributions.

The rest of the paper is organized as below. State-of-the-art relevant works are discussed in Sec.~\ref{RelatedWork}. Preliminary knowledge regarding model average algorithms and compression algorithms are introduced in Sec.~\ref{Preliminary}.  We then present the convergence rate analysis in Sec.~\ref{ConvergenceRate}, together with the rate-optimized adaptive compression framework in Sec.~\ref{AdaptingCompression}. We discuss the experiment results in Sec.~\ref{Experiment} and conclude the paper in Sec.~\ref{Conclusion}.

\section{Related Works} \label{RelatedWork}
In this section, we discuss related works in Federated Learning (FL) from two aspects, namely, model averaging and model compression. The former is the foundation for FL, and the later is essential for minimizing traffic overhead in FL. 

\subsection{FL Model Averaging}

In Federated Learning (FL) \cite{lim2020federated, yang2019federated}, decentralized clients collaboratively train machine learning models via {\em model averaging, i.e.,} iteratively averaging  locally  trained  models towards global optimum \cite{mcmahan2017communication}. 
%There have been significant studies on the model averaging algorithms~\cite{povey2015parallel}, and Federated Averaging (FedAvg) \cite{mcmahan2017communication} is arguably the most widely used to date. 
%In \cite{}, the authors introduced some information of FL and show a secure FL framework. %In \cite{lim2020federated}, the author introduced a comprehensive survey of FL and discussed the problems existing in the implementation of FL.
%Based on the feature of FL, researchers applied FL in some fields of concern for privacy, such as medical \cite{brisimi2018federated} and politics \cite{verma2018federated}. %Google has been applied FL to Gboard to improve the performance of the word prediction model \cite{hard2018federated}.
It runs Stochastic Gradient Descent (SGD) in parallel on a subset of devices and then averages the resulting model updates via a central server once in a while. The convergence rates of FedAvg have been analyzed in \cite{li2019convergence} with strongly convex loss functions and \cite{yang2021achieving} with non-convex loss functions, respectively. There have also been variants of FedAvg that target different application scenarios. 
Wang \emph{et al.} in \cite{wang2019adaptive} designed a control algorithm to balance local updates and global aggregation in FL under limited resource in edge computing; 
FEDL \cite{tran2019federated}  optimizes the allocation of resources in wireless networks to balance the convergence and resource consumption. 
%Luo \emph{et al.} \cite{luo2020cost} further established the relationship between the energy cost and FL convergence, and presented an adaptive algorithm that minimizes the energy cost with guaranteed convergence. 
Our work focuses on the classical FedAvg with model compression; yet our solution can be generalized to work with these extended versions. 

\subsection{FL Model Compression}

Transmitting model updates in a large scale network incurs heavy traffic overhead. The original FedAvg  \cite{mcmahan2017communication} increases the number of local iterations on clients to minimize the costly global communications. 
%In \cite{yao2018two}, each client tries to limit the difference between its local model and the global model so that the number of communications can be reduced. 
%Hierarchical \cite{liu2020client} or p
Parallel FL architecture \cite{zhong2021} seeks to reduce the network bandwidth consumption through using multiple servers. Our work however focuses on model compression \cite{haddadpour2021federated, cui2021slashing}, which is orthogonal to them, and can work together with them to maximize traffic reduction.

Model compression in FL can be achieved through two methods \cite{shi2019convergence}, namely,  {\em sparsification} and {\em quantization}.

Sparsification algorithms select only a small number of significant model updates for transmission.
%In \cite{wangni2018gradient}, Wangni \emph{et al.} designed an unbiased sparsification algorithm and proposed a convex optimization problem to minimize the encoding length.
%Stich \emph{et al.} defined $k$-sparsification in \cite{stich2018sparsified} and analyzed the convergence of the SGD algorithm in distributed learning. % combined with error tracking according to the given definition. 
The DGC algorithm proposed in \cite{lin2017deep} discards 99.9\% of gradients in communication, achieving a very high compression rate. %, which however can severely compromise model accuracy. }
Gaia \cite{hsieh2017gaia} dynamically eliminates insignificant communications, and
STC \cite{sattler2019robust} compresses uploaded and downloaded data simultaneously through sparse and ternary methods. %, thus optimizing the communication. 
DC2 \cite{abdelmoniem2021dc2} further explores the delay-traffic trade off through adaptively coupling compression control and network latency. 
%In \cite{han2020adaptive}, Han \emph{et al.}, using online learning to select optimal sparsity, further proposed a fair bidirectional sparse algorithm. %according to a novel online learning formula.
Although sparsification can be quite effective,  the solutions above are mostly biased, implying that they may not guarantee the convergence of FL and hence is not our focus in this work. 

Quantization algorithms map model updates to a small set of discrete values. For instance, the QSGD algorithm \cite{alistarh2017qsgd} generates a random number based on each model update and map each update to a centroid.  The PQ algorithm \cite{suresh2017distributed} splits model updates into intervals with a number of centroids, and each model update is randomly quantified to a centroid in an unbiased manner. 
In \cite{wen2017terngrad}, Wen \emph{et al.}  optimized communication by quantizing model updates to ternary values.
%The SignSGD algorithm \cite{bernstein2018signsgd} further suggests to transmit only the sign of each model update, which was proven to achieve the convergence rate as the SGD algorithm. 
The mapping in most quantization algorithms is unbiased. Yet most of them adopt a fixed compression rate because the relation between the compression error and the learning rate is largely unknown.
In this paper,  we for the first time analyze the effect of the compression on the model convergence for both strongly convex and non-convex cases. We accordingly propose an adaptive compression framework with optimal rate allocation, which seeks the balance between model accuracy and communication overhead for FL over real world networks.

%Although most of quantization algorithms are unbiased to ensure the convergence of the training, the compression rate of each communication round is fixed in these algorithms.
%Therefore, the compression rate cannot be adjusted adaptively according to the influence of learning rate on compression error in each communication round.

%In this paper, we will analyze the effect of the compression on the model convergence for the unbiased quantization algorithm in both convex and non-convex cases. According to the analysis results, the number of centroid of each communication round is selected adaptively to minimize the impact of compression.

\section{Background and System Model} \label{Preliminary}
In this section, we introduce the necessary background on Federated Learning, particularly on the model averaging algorithm, \emph{i.e.}, FedAvg \cite{mcmahan2017communication}, and the compression algorithms. We also outline our system model with adaptive compression and summarize the key notations. 

\subsection{Federated Learning Basics}
In FL, data samples are distributed on $N$ clients, where the data set owned by client $i$ is $\mathcal{D}_i$. All $N$ clients collaboratively train a machine learning model, and their local results are aggregated through {\em model averaging}, typical through FedAvg or its variants \cite{wang2020federated}. A {\em loss function} of the training model on client $i$ can be denoted by:
\begin{eqnarray}
    F_i(\mathbf{w}, \mathcal{D}_i) = \frac{1}{|\mathcal{D}_i|} \sum_{\forall \xi\in\mathcal{D}_i} f(\mathbf{w}, \xi).
    \label{EQ:LocalLoss}
\end{eqnarray}
Here, $\mathbf{w}\in \mathbb{R}^d$  with dimension $d$ is the model parameters to be learned,  $|\mathcal{D}_i|$ denotes the size of client $i$'s dataset, and $f(\mathbf{w}, \xi)$ denotes the loss function calculated by a specific sample $\xi$.
The goal of FL is to train the model parameters such that the global loss function can be minimized, that is
\begin{eqnarray}
    \min_{\mathbf{w}}\quad F(\mathbf{w}) = \sum_{i=1}^N p_i F_i(\mathbf{w},\mathcal{D}_i),
    \label{EQ:GlobalLoss}
\end{eqnarray}
where $p_i$ represents the weight of client $i$, which is typically set to $p_i = \frac{|\mathcal{D}_i|}{\sum_{i'=1}^N|\mathcal{D}_{i'}|}$.

FedAvg conducts multiple global iterations to iteratively reduce the global loss function $F(\mathbf{w})$. In global iteration $t$,  a {\em Parameter Server} (PS)  randomly selects $K$ clients, denoted by $\mathcal{K}_t$, to participate in model training. 
Each selected client downloads the latest model parameters denoted by $\mathbf{w}_t$ from the PS, and then conducts $E$-round local iterations with  the local data samples. Each round works as follows:
\begin{eqnarray}
    \label{EQ:LocalTrain}
    \mathbf{w}^i_{t,j+1} = \mathbf{w}^i_{t,j}-\eta_{t} \nabla F_i(\mathbf{w}^i_{t,j},\mathcal{B}^i_{t,j}),
\end{eqnarray}
where $j$ represents the $j^{th}$ local iteration, $\eta_t$ represents the learning rate at global iteration $t$, and $\mathcal{B}^i_{t,j}$ represents the sample batch with size $B$ randomly selected by client $i$ for this local iteration. 
After the $E$ local iterations, participating clients will send model updates $\mathbf{U}^i_t = \sum_{j=0}^{E-1} \nabla F_i(\mathbf{w}^i_{t,j},\mathcal{B}^i_{t,j})$ back to the PS for aggregation, as follows
\begin{eqnarray}
    \label{EQ:GlobalAggregation}
    \mathbf{w}_{t+1} =\mathbf{w}_{t} \!\!-\!\! \eta_{t}\sum_{i\in \mathcal{K}_{t}}\frac{|\mathcal{D}_i|}{\sum_{i^{'}\in\mathcal{K}_{t}}|\mathcal{D}_{i^{'}}|}\mathbf{U}^i_t.
\end{eqnarray}
% \begin{eqnarray}
%     \label{EQ:GlobalAggregation}
%     \mathbf{w}_{t+1} =\mathbf{w}_{t} \!\!-\!\! \sum_{i\in \mathcal{K}_{t}}\frac{|\mathcal{D}_i|}{\sum_{i^{'}\in\mathcal{K}_{t}}|\mathcal{D}_{i^{'}}|}\big(\sum_{j=0}^{E-1}\eta_{t} \nabla F_i(\mathbf{w}^i_{t,j},\mathcal{B}^i_{t,j})\big).
% \end{eqnarray}
% Let $\mathbf{U}^i_t = \sum_{j=0}^{E-1} \nabla F_i(\mathbf{w}^i_{t,j},\mathcal{B}^i_{t,j})$, which is the model updates uploaded by client $i$ at global iteration $t$.

\subsection{Compression with Rate Adaptation}

Uploading model updates $\mathbf{U}^i_t $ to the PS can be very time consuming in dynamic network environment given the limited upload bandwidth \cite{konevcny2016federated}. Compression can be applied to $\mathbf{U}^i_t $, so as to reduce the traffic. As mentioned in the previous section, we focus on quantization-based compression, which uses a small number of centroids to represent model updates. Let $\widetilde{\mathbf{U}}^i_t$ denote the compression of $\mathbf{U}^i_t $. To avoid compromising the convergence rate of FL, $\widetilde{\mathbf{U}}^i_t$ is typically set as a random variable with expectation $\mathbb{E}\widetilde{\mathbf{U}}^i_t=\mathbf{U}^i_t$ \cite{alistarh2017qsgd}. In other words, $\widetilde{\mathbf{U}}^i_t$ is the unbiased estimation of $\mathbf{U}^i_t$.\footnote{There are also biased compression algorithms that do not satisfy $\mathbb{E}\widetilde{\mathbf{U}}^i_t=\mathbf{U}^i_t$, \emph{e.g.}, \cite{luping2019cmfl}. Though being efficient in certain application scenarios, the model accuracy trained with such biased algorithms cannot be guaranteed and hence is not our focus in this paper} Let $\widetilde{\mathbf{U}}^i_t$ be a discrete random variable with only $Z$ possible values, denoted by $u_1,u_2,\dots,u_Z$, which are referred to as $Z$ centroids for compression.

%We use $\widetilde{\mathbf{U}}^i_t$ to represent the compression of the uploaded data. In order to ensure the convergence of the model, we consider those unbiased quantization compression algorithms in this article. Through these compression algorithms, the model update will be approximated by $Z$ discrete values, $u_1,u_2,\dots,u_Z$, in an unbiased manner. Therefore, the uploaded data becomes a random vector represented by $Z$ centroids. The expected value is equal to the original value, \emph{i.e.,} $\mathbb{E}\widetilde{\mathbf{U}}^i_t=\mathbf{U}^i_t.$ 

Figure~\ref{process} outlines the process to transmit model updates from a particular client to the PS with a fixed compression rate and adaptive compression rates, respectively. 
%\textcolor{blue}{We use the width of the arrows to indicate the amount of data uploaded by clients, implying different compression rates.} 
Though the former, for its simplicity, has been adopted in most existing works on quantization, our analysis in the next section suggests that the compression rate should be adjusted according to the influence of the compression error in each communication round in a real world network environment with resource constraints and dynamics, and we accordingly develop the optimal adaption strategy that maximizes the final model accuracy.  

\begin{figure}[h]
    \centering
    \includegraphics[width=0.9\linewidth]{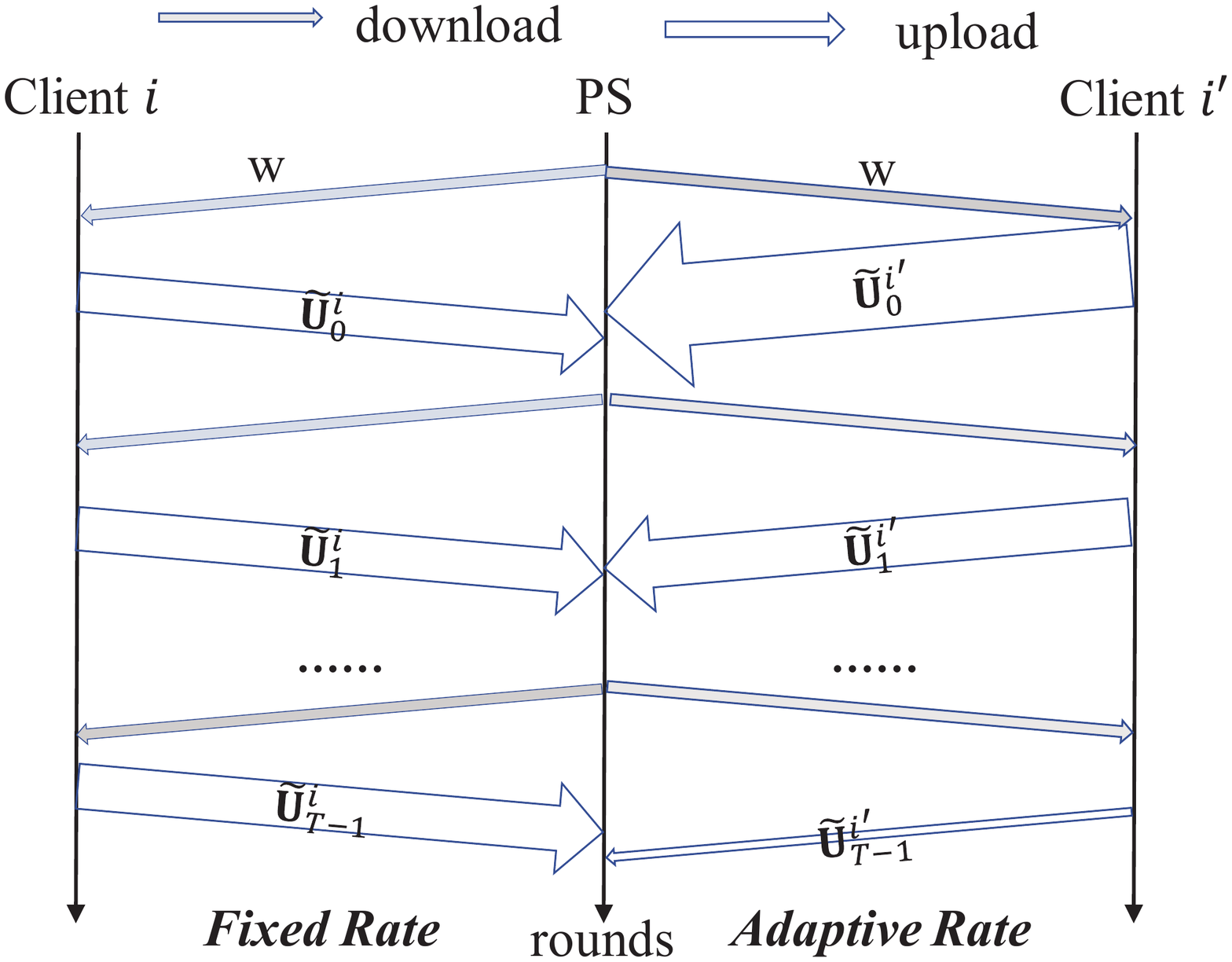}
    \caption{Transmitting model updates with a fixed compression rate (left) \emph{vs}  adaptive compression rates (right): the width of arrows indicates the amount of data uploaded by the client, implying different compression rates.}
    \label{process}
\end{figure}

In Table~\ref{NotationList}, we summarize the key notations used in this paper. Without loss of generality, we define the compression error as $J_t^i = \mathbb{E}\|\widetilde{\mathbf{U}}^i_t-\mathbf{U}^i_t\|^2$. Note that, this definition is applicable for all unbiased compression algorithms through customizing $J_t^i$. The influence of $J_t^i$ on the final model accuracy will be quantified in the next section through our convergence analysis with model compression.

\begin{table}[h]
    \centering
    \caption{Notation List}
    \begin{tabular}{|m{1cm}<{\centering}|m{6.5cm}<{\centering}|}
         \hline
         Notation & Meaning \\
         \hline
         $p_i$ & the weight of client $i$ among all clients \\
         \hline
         $\mathbf{U}^i_{t}$ & model updates of client $i$ in the $t^{th}$ global iteration\\
         \hline
         $\widetilde{\mathbf{U}}^i_{t}$ & compressed model updates randomly generated based on $\mathbf{U}^i_{t}$\\
         \hline
         $Z$ & number of centroids\\
         \hline
         $d$ & dimensions of the model \\
         \hline
         $\mathcal{K}_{t}$ & set of clients selected to participate the $t^{th}$ global iteration\\
         \hline
         $\Gamma_c/\Gamma_n$ & quantified the heterogeneity of the non-IID data distribution if loss function is strongly convex/non-convex\\
        %  \hline
        %  $\Lambda$ & compression rate\\
         \hline
         $\mathbf{w}^i_{t,j}$ & model updated $j$ times on  client $i$ in the $t^{th}$ global iteration \\
         \hline
         $\mathbf{w}_{t}$ & model downloaded from PS to clients in the $t^{th}$ global iteration \\
         \hline
    \end{tabular}
    \label{NotationList}
\end{table}

%\begin{eqnarray}
 %   \label{EQ:DefinitionOfJ}
 %   &&\eta_t\mathbb{E}\left\|\widetilde{\mathbf{U}}^i_t-\mathbf{U}^i_t\right\|^2=\eta_t^2J^i_t
    % &&=\eta_{t}^2\mathbb{E}\left\|\widetilde{\sum_{j=0}^{E-1} \nabla F_i(\mathbf{w}^i_{t,j},\mathcal{B}^i_{t,j})}-\sum_{j=0}^{E-1} \nabla F_i(\mathbf{w}^i_{t,j},\mathcal{B}^i_{t,j})\right\|^2\notag\\
    % &&=\eta_t^2J^i_t.
%\end{eqnarray}

%\subsection{Adaptive Compression}

\section{ Convergence Analysis under Model Compression} \label{ConvergenceRate}

In this section, we analyze the convergence rate of FedAvg when  compressed model updates are transmitted to the PS. As in previous works  \cite{li2019convergence, yu2019parallel, haddadpour2021federated, yang2021achieving}, we make the following assumptions on the learned models in FL.

\begin{assumption}
\label{Assump:Smooth}
All loss functions, \emph{i.e.}, $F_1, F_2,\dots, F_N$ are $L$-smooth; that is, given $\mathbf{v}$ and $\mathbf{w}$, we have $F_i(\mathbf{v}) \le F_i(\mathbf{w}) + (\mathbf{v}-\mathbf{w})^T\nabla F_i(\mathbf{w})+\frac{L}{2}||\mathbf{v}-\mathbf{w}||^2$.
\end{assumption}

\begin{assumption}
\label{Assump:LocalVar}
Let $\xi^i_t$ denote a sample randomly and uniformly selected from client $i$. The variance of the stochastic gradients in each client is bounded:
$\mathbb{E}[\|\nabla F_i(\mathbf{w}^i_{t,j}, \xi^i_{t,j})-\nabla F_i(\mathbf{w}^i_{t,j})\|^2] \le \sigma^2$ for $\forall i,\forall j, \forall t$.
%for $i=1, 2,\dots, N$.
\end{assumption} 

\begin{assumption}
\label{Assump:BoundG}
The expected square norm of stochastic gradients is uniformly bounded, \emph{i.e.},  $\mathbb{E}[\|\nabla F_i(\mathbf{w}^i_{t,j}, \xi^i_{t,j})\|^2] \le G^2$ for $\forall i,\forall j, \forall t$.
%for all $i = 1,2,\dots,N$, $j=0, 1,\dots, E-1$ and $t = 0, 1,\dots, T-1$
\end{assumption}

These assumptions hold in typical  FL models, as discussed in the previous works  \cite{li2019convergence, yu2019parallel, haddadpour2021federated, yang2021achieving}. We also assume that the FL system uses a {\em Partial Client Participation Mode} \cite{li2019convergence}, \emph{i.e.,} in each round of global iteration, the PS randomly selects $K$ clients  according to the weight probabilities $p_1, p_2, \dots, p_N$ with replacement  to conduct local iterations. The set of selected clients is denoted by $\mathcal{K}_t$ at the $t^{th}$ global iteration. The PS then aggregates model updates by $\mathbf{w}_t = \mathbf{w}_{t-1} - \frac{\eta_{t-1}}{K}\sum_{i\in\mathcal{K}_{t-1}}\widetilde{\mathbf{U}}^i_{t-1}$, where $\widetilde{\mathbf{U}}^i_{t-1}$ denotes compressed model updates. 

%  {\color{red} better to have a notation list here,
 
%  try to make your related work and appendix more concise, otherwise it is over-length. 
%  }

\subsection{With Strongly Convex Loss  Functions}

We first analyze the convergence rate by assuming that all loss functions $F_1, F_2,\dots, F_N$, are $\mu$-strongly convex; that is, given $\mathbf{v}$ and $\mathbf{w}$, we have $F_i(\mathbf{v}) \ge F_i(\mathbf{w}) + (\mathbf{v}-\mathbf{w})^T\nabla F_i(\mathbf{w})+\frac{\mu}{2}||\mathbf{v}-\mathbf{w}||^2$. In practice, most loss functions belong to this category \cite{li2019convergence, dinh2020federated}.

%\begin{definition} (Quantification of non-IID)
For strongly convex loss functions, we use $\Gamma_c = F^*-\sum_{i=1}^Np_iF^*_i$ to quantize the degree of non-IID sample distribution on clients. Here, $F^*$ and $F^*_i$ are the optimal values of $F$ and $F_i$, respectively. 
%\end{definition}

\begin{theorem}
\label{THE:ConvexCR1}
Let $\eta_t=\frac{2}{\mu(t+\gamma)}$, $\gamma=\frac{8L}{\mu}$, $\kappa=\frac{L}{\mu}$ and $\mathbf{w}^*$ is the optimal model, the convergence rate of FedAvg with  compressed unbiased model updates is 
\begin{align} 
   \begin{autobreak} \MoveEqLeft[0]
    \mathbb{E}\left\|\mathbf{w}_{T} - \mathbf{w}^*\right\|^2\le
  \end{autobreak}\nonumber\\
   \begin{autobreak} \MoveEqLeft[0]
  \frac{4}{\gamma+T}\Big(\frac{\alpha}{\mu}+\sum_{t=0}^{T-1}\frac{\eta_t\sum_{i=1}^Np_iJ^i_t}{2K\mu}+2\kappa\left\|\mathbf{w}_{0} - \mathbf{w}^*\right\|^2\Big),
  \end{autobreak}
 \end{align}

% \begin{eqnarray}
%     \mathbb{E}\left\|\bar{\mathbf{w}}_{T} - \mathbf{w}^*\right\|^2\le\frac{v_T}{\gamma+T},
% \end{eqnarray}
where $\alpha=\frac{E\sum_{i=1}^Np_i^2\sigma^2}{B}+6EL\Gamma_c+2E(E-1)^2G^2+\frac{E^2G^2}{K}$.
%and $\psi_{t}= \frac{1}{K}\sum_{i=1}^Np_{i}\eta_t^2 J^{i}_{t}$.

%is decreasing and $\eta_t\le2\eta_{t+1}$, %then according to Lemmas \ref{Lemma:v-w}-\ref{Lemma:PartialClientsComp} 
% we can derive the convergence rate of FedAvg with compressed model updates as 
% \begin{eqnarray}
%     &&\mathbb{E}\left\|\bar{\mathbf{w}}_{T} - \mathbf{w}^*\right\|^2\notag\\
%     &&\le\sum_{t=0}^{T-1}\eta_t^2\left(\frac{E^2G^2}{K} + \frac{\sum_{i=1}^Np_{i} J^{i}_{t}}{K}\right) \notag\\
%     &&\qquad+ \sum_{t=0}^T\eta_t\frac{\tau-\tau(1-\eta_t\mu)^E}{\mu}\notag\\% \eta_t^2\tau\frac{1-a^{T*E}}{1-a} \notag\\
%     &&\qquad +\left(\prod_{t=0}^T(1-\eta_t\mu)\right)^E\left\|\bar{\mathbf{w}}_{0} - \mathbf{w}^*\right\|^2
% \end{eqnarray}
% where $\tau=\frac{E\sum_{i=1}^Np_i^2\sigma^2}{B}+6EL\Gamma+2E(E-1)^2G^2$.
% \begin{eqnarray}
%     &&\mathbb{E}\left\|\bar{\mathbf{w}}_{T} - \mathbf{w}^*\right\|^2\notag\\
%     &&\le\sum_{t=1}^{T}\frac{1}{(\mu t)^2}\left(\frac{E^2G^2}{K} + \frac{\sum_{i=1}^Np_{i} J^{i}_{t}}{K}\right) \notag\\
%     &&\qquad+ \sum_{t=1}^T\frac{\tau-\tau(1-\frac{1}{t})^E}{\mu^2t}\notag\\% \eta_t^2\tau\frac{1-a^{T*E}}{1-a} \notag\\
%     %&&\qquad +\left(\prod_{t=0}^T(1-\eta_t\mu)\right)^E\left\|\bar{\mathbf{w}}_{0} - \mathbf{w}^*\right\|^2\notag
% \end{eqnarray}
\end{theorem}
The key to the proof is to construct a model for each local iteration and derive the gap between this model and the optimal model. Please refer to Appendix \ref{ProofOfTheorem1} for the detailed proof. 

It is worth mentioning that,  the compression error can be bounded by $J_t^i = \mathbb{E}\|\widetilde{\mathbf{U}}^i_t-\mathbf{U}^i_t\|^2\leq (1-\delta)\|\mathbf{U}_t^i\|$, where 
%the compression is a $\delta$-contraction operator and  
$0<\delta\leq 1$ \cite{stich2018sparsified}. With this bound, the convergence rate $O(\frac{1}{T})$ was proved in the work \cite{cui2021slashing}. However, this bound ignores the relation between the number of centroids and the compression error,  and thereby we cannot adapt compression rate based on the  convergence rate derived with this bound.

\noindent{\bf Remark:} 
We can extract term $\sum_{t=0}^{T-1}\eta_t\sum_{i=1}^Np_{i} J^{i}_{t}$ from the convergence rate, which indicates that the influence of the compression error on the final model accuracy depends on the learning rate. To maximize the final model accuracy, we should minimize $\sum_{t=0}^{T-1}\eta_t\sum_{i=1}^Np_{i} J^{i}_{t}$, and accordingly adjust the compression rates in accordance with the learning rates.

\subsection{With Non-convex Loss Functions}

We proceed to analyze the convergence rate when loss functions are non-convex. In this case, we use the difference between global gradients and local gradients to quantize the non-IID degree of the  sample distribution on clients. That is $\Gamma_n \geq \mathbb{E}\|\nabla F_i(\mathbf{w}_{t})-\nabla F(\mathbf{w}_{t})\|^2 $.

\begin{theorem}
\label{THE:NonConvexCR2}
    Let $c>0$ is a constant and a learning rate satisfying $\eta_t\le\frac{1}{8LE}$,  $\eta_tEL\le\frac{K}{K-1}$ and $30NE^2\eta_t^2L^2\sum_{i=1}^Np_i^2+\frac{L\eta_t}{K}(90E^3L^2\eta_t^2+3E)<1$, the convergence rate of FedAvg with  compressed unbiased model updates is
%     \begin{eqnarray}
% 	&&\min_{t\in[T]}\mathbb{E}\|\nabla F(\mathbf{w}_t)\|^2\notag\\
% 	&&\le \frac{L}{2c\eta_{T-1}TE}\sum_{t=0}^{T-1}\eta_t^2\sum_{i=1}^Np_iJ^i_t+\frac{(\sigma^2+6E\Gamma_n)}{c\eta_{T-1}TE}\notag\\
% 	&&\Big[\frac{5NE^2L^2\sum_{i=1}^Np_i^2\sum_{t=0}^{T-1}\eta_t^3}{2}+\frac{15E^3L^3\sum_{t=0}^{T-1}\eta_t^4}{2K}\Big]  \notag\\
% 	%&&+\frac{L}{2c\eta_TTE}\sum_{t=0}^T\eta_t^2\sum_{i=1}^Np_iJ^i_t\notag\\
% 	&&+\Big(\frac{LE^2\sigma^2}{2}+\frac{3E^2L\Gamma_n}{2K} \Big)\frac{\sum_{t=0}^{T-1}\eta_t^2}{c\eta_{T-1}TE}+\frac{F_0-F_*}{c\eta_{T-1}TE},
%     \end{eqnarray}
    \begin{align} 
        \begin{autobreak} \MoveEqLeft[0]
            \min_{t\in[T]}\mathbb{E}\|\nabla F(\mathbf{w}_t)\|^2\le \frac{L}{2c\eta_{T-1}TE}\sum_{t=0}^{T-1}\eta_t^2\sum_{i=1}^Np_iJ^i_t
        \end{autobreak}\nonumber\\
        \begin{autobreak} \MoveEqLeft[0]
            +(\frac{LE^2\sigma^2}{2}+\frac{3E^2L\Gamma_n}{2K} )\frac{\sum_{t=0}^{T-1}\eta_t^2}{c\eta_{T-1}TE}+\frac{(\sigma^2+6E\Gamma_n)}{c\eta_{T-1}TE}
        \end{autobreak}\nonumber\\
        \begin{autobreak}\MoveEqLeft[0]
            \Big[\frac{5NE^2L^2\sum_{i=1}^Np_i^2\sum_{t=0}^{T-1}\eta_t^3}{2}+\frac{15E^3L^3\sum_{t=0}^{T-1}\eta_t^4}{2K}\Big]
        \end{autobreak}\nonumber\\
        \begin{autobreak} \MoveEqLeft[0]
            +\frac{F_0-F_*}{c\eta_{T-1}TE},
        \end{autobreak}
    \end{align}
\end{theorem}

where $F_0$ and $F_*$ represent the initial value and optimal value of the global loss function, respectively. The key to the proof is to analyze the decrease of the loss function after each global iteration through the $L-smooth$ property. Please refer to Appendix \ref{ProofOfTheorem2} for the detailed proof.

\noindent{\bf Remark:}  It can be seen that the influence of compression error is related with $\eta_t^2$. To maximize the final model accuracy, we should minimize $\sum_{t=0}^{T-1}\eta_t^2\sum_{i=1}^Np_iJ^i_t$. Again, although the convergence rate $O(\frac{1}{\sqrt{T}})$ can be found in \cite{gao2021convergence}, it ignores the relation between the compression error and the number of centroids and hence does not work for rate adaption.

\section{Optimal Rate Adaptation for Model Compression} \label{AdaptingCompression}

The  communication overhead of each global iteration can be quantified by the number of centroids $Z$. 
Suppose that each original model update takes  $h$ bytes. With a compression algorithm, \emph{e.g.,} Probability Quantification (PQ) \cite{konevcny2016federated}, it takes $Zh$ bytes to represent the $Z$ centroids and $\log_2{Z}$ bits to represent the identification of each model update. The total traffic will be $Zh+d\log_2{Z}$, where $d$ is the model dimension. For high-dimensional models, we have  $d\gg Z$ \cite{seide20141}.  The traffic of each global iteration is therefore approximately $\frac{d\log_2{Z}}{8}$ bytes. 
Assume transmitting the original model updates needs $hd$ bytes, the compression rate is given by $\Lambda= \frac{hd}{d\frac{\log_2{Z}}{8}}=\frac{8h}{\log_2{Z}}$. %For different compression algorithms, the compression rate 

Based on the derived convergence rates, we now establish the policy to maximize the final model accuracy by adapting  compression rates in accordance with learning rates.

%. Each element in the random vector after compression only needs to use $\frac{log_2Z}{8}$ bytes. Therefore, the compression rate achieved by the unbiased quantization algorithm is $\Lambda=\frac{hd}{\frac{log_2Zd}{8}+Zh}$.However, compression error will inevitably be introduced in this process and we define the error caused by compression as follows

\subsection{Optimal Rate Adaptation Framework}

We formulate the problem to optimize model accuracy by adapting compression rates without incurring additional communication overhead.

From Theorems~\ref{THE:ConvexCR1} and \ref{THE:NonConvexCR2}, we can extract  terms related with compression errors, and define the adaptive rate objective as:
\begin{equation}
    \mathcal{J}=\left\{
        \begin{aligned}
        &\sum_{t=0}^{T-1}\eta_t\sum_{i=1}^Np_iJ^i_t, \quad \textit{Loss $F_i$ is strongly convex},\\
        %&\mathbf{w}_{t+1-E}^i - comp\left(\sum_{i=1}^Np_i\widetilde{\mathbf{H}}^i_{t+1}\right), \quad if\ t+1\ \in \mathcal{I}. \\
        &\sum_{t=0}^{T-1}\eta_t^2\sum_{i=1}^Np_iJ^i_t, \quad \textit{Loss $F_i$ is non-convex}. \\
        \end{aligned}
        \right.
\end{equation}

%$\mathcal{J}=\sum_{t=0}^T\eta_t^2\sum_{i=1}^Np_iJ^i_t$. 
Here $J^i_t$ is the compression error on client $i$ in the $t^{th}$ global iteration.  To maximize the final model accuracy with a finite $T$, we should minimize $\mathcal{J}$. 
%where compression error $J^i_t$ is a function affected by the number of centroids $Z^i_t$. 
Recall that the compression error is defined as  $J^i_t = \mathbb{E}\|\widetilde{\mathbf{U}}^i_t-\mathbf{U}^i_t\|^2$.  
%XXXX is it $J^i_t$ or $J_t^i$? XXX
Thus, $J_t^i$ is a function that will be affected by the number of centroids $Z_t$, where $Z_t$ should be a positive integer.

Let $d\times x_t$ bits be the approximate traffic with compressed model updates in the $t^{th}$ global iteration where $x_t = \log_2{Z_t}$.\footnote{We suppose that the compression rate is identical across clients in each global iteration} Let $\mathbf{x} = (x_0, \dots, x_{T-1})$,
%represent the traffic in the $T$ global iterations. 
the total traffic will be $d\times\|\mathbf{x}\|_1$, which can be bounded by a constant $C$.

%From the analysis results of Theorem 1 and Theorem 2, we can draw the conclusion that whether it is for a non-convex loss function or a strongly convex loss function, the impact of the compression error on the model convergence is $\sum_{t=0}^T\eta_t^2\sum_{i=1}^Np_iJ^i_t,$ where compression error $J^i_t$ is a function affected by the number of centroids $Z^i_t$.

%In order to facilitate the solution of the above problem, we assume that the number of centroids used by each client is the same, \emph{i.e.,} $Z_t=Z^1_{t}=Z^2_{t}=\dots=Z^N_{t}$ and $J_t=J^1_t=J^2_t=\dots=J^N_t$. In addition, we use the temporary value $x_t=log_2Z_t$ to represent the number of bits used by the clients to represent each model uodate in the $t^{th}$ round of communication and $J_t$ become the function of $x_t$. Finally, we can get the optimization problem as follow
We then have the following integer programming problem:
\begin{eqnarray}
\label{EQ:optimV1}
	 \mathbb{P}1: &&\min_{Z_0, \dots, Z_{T-1}} \mathcal{J}\notag\\
	&&s.t. \quad \|\mathbf{x}\|_1 \le \frac{C}{d}\notag\\
	&& \qquad Z_{t}=\{1, 2, \dots\} \quad \textit{for $t=0, \dots, T-1$.}%\notag\\&& \qquad x_{t}\le32
\end{eqnarray}
Due to the difficulty to solve an integer programming problem, we relax the constraint to allow $Z_t$ to be a real positive number. Then, we can formulate the following problem. 
\begin{eqnarray}
\label{EQ:optim}
	\mathbb{P}2: &&\min_{Z_0, \dots, Z_{T-1}} \mathcal{J}\notag\\
	&&s.t. \quad \|\mathbf{x}\|_1 \le \frac{C}{d}\notag\\
	&& \qquad Z_{t}\geq 1 \quad \textit{for $t=0, \dots, T-1$.}%\notag\\&& \qquad x_{t}\le32
\end{eqnarray}
Since $Z_t$ is a real number, we can solve  $\mathbb{P}2$ through gradient descent algorithm \cite{lee2016gradient},  and then each solved $Z_t$ can be rounded to its nearest integer.

\subsection{Implementation Issues}

To implement our adaption framework, such information as $\eta_t$, $p_i$, and $C$ are to be obtained by the PS. The learning rate $\eta_t$ can be reported by clients, which only takes 4 bytes in each communication round; The weight $p_i$ (usually proportional to the sample population owned by client $i$) and  the restriction of communication traffic $C$ can be determined prior to the model training.  After solving problem $\mathbb{P}2$, the PS can send out the adaptive compression rate  $Z_t$ (a single number) together with the latest model parameters to participating clients in each communication round.  

Our framework is applicable with various unbiased compression algorithms in FL. In practice, we need to derive the expression of $J_t^i$ with specified compression algorithms. To demonstrate the usability of our framework, we conduct a case study with two widely used compression algorithms for FL, namely, Probability Quantization \cite{suresh2017distributed} and Quantized SGD \cite{alistarh2017qsgd}. 

\subsubsection{Case Study with Probability Quantization (PQ)}
The PQ algorithm uses $Z$ centroids to split a certain range into $Z-1$ intervals. The model update in an interval will be probabilistically quantified in an unbiased manner to the  upper or lower centroids  of the interval. 
% We use $\mathcal{C}_{PQ}(\mathbf{a})$ to indicate that the vector $\mathbf{a}$ is compressed using the PQ algorithm. According to \cite{suresh2017distributed}, we can get
%$$\mathbb{E}\|\mathcal{C}_{PQ}(\mathbf{a})-\mathbf{a}\|^2\le\frac{d(a_{max}-a_{min})^2}{4(Z-1)^2}\le\frac{d\|\mathbf{a}\|^2}{2(Z-1)^2},$$
% Where $d$ is the dimension of the vector and $a_{max}/a_{min}$ is the maximum/minimum value in the vector. 
According to \cite{suresh2017distributed}, the compression error can be bounded as follows.
%According to Eq.~\eqref{EQ:DefinitionOfJ} and Assumption~\ref{Assump:BoundG}, we can get the compression error of the PQ algorithm as
\begin{eqnarray}
\label{EQ:PQJt}
J^i_t&=&\mathbb{E}\left\|\widetilde{\mathbf{U}}^i_t-\mathbf{U}^i_t\right\|^2\notag\\
&\overset{(1)}{\le}& \frac{d(U_{max}-U_{min})^2}{4(Z-1)^2}\notag\le \frac{d(2U^2_{max}+2U^2_{min})}{4(Z-1)^2}\notag\\
&\le& \frac{2d\|\mathbf{U}_t^i\|^2}{4(Z-1)^2}= \frac{d\|\sum_{j=0}^{E-1} \nabla F_i(\mathbf{w}^i_{t,j},\mathcal{B}^i_{t,j})\|^2}{2(Z_t-1)^2}\notag\\
&\overset{(2)}\le& \frac{dE^2G^2}{2(Z_t-1)^2}=\frac{dE^2G^2}{2(2^{x_t}-1)^2},
\end{eqnarray}
where $U_{max}$ and $U_{min}$ are the maximum and minimum values of each element in vector $\mathbf{U}_t^i$, respectively. Here, inequality $(1)$ is adapted from Theorem 2 in \cite{suresh2017distributed} and inequality $(2)$ is from Assumption~\ref{Assump:BoundG}.

It can be verified that $J_t^i$ defined in Eq.\eqref{EQ:PQJt} is convex with respect to $x_t$. Substituting Eq.\eqref{EQ:PQJt} into $\mathbb{P}2$, we can solve $x_t$, and hence $Z_t$.

% Substituting the compression error of the PQ algorithm into problem~\eqref{EQ:optim}, the optimization problem will become as follows
% \begin{eqnarray}
% 	&&\min \sum_{t=1}^{T}\frac{\eta_t^2}{(2^{x_t}-1)^2}\notag\\
% 	&&s.t. \quad x_{1}+x_{2}+\dots+x_{T} \le C\notag\\
% 	&& \qquad x_{t}\ge1\notag\\
% 	&& \qquad x_{t}\le32.
% \end{eqnarray}

\subsubsection{Case Study with Quantized SGD (QSGD)}

Quantized SGD (QSGD) is another popular unbiased compression algorithm in FL. Let $\mathbf{U}$ denote an arbitrary model update vector and $U_e$ denote an arbitrary element in $\mathbf{U}$.  $U_e$  is quantified to  $Q(U_e)$ by 
$Q(U_e)=\|\mathbf{U}\|^2*sgn(U_e)*\xi_e,$
where $sgn(U_e)$ is the sign of $U_e$ and $\xi_e$ is a random variable affected by $U_e$. We can control $\xi_e$ such that $\mathbb{E}[Q(U_e)] = U_e$, and  the compression error of QSGD is bounded in \cite{alistarh2017qsgd}, which is
%QSGD quantifies model updates by rounding each model update to one of $Z$ discrete values. 
%We use $\mathcal{C}_{QSGD}(\mathbf{a})$ to indicate that the vector $\mathbf{a}$ is compressed using the QSGD algorithm. According to \cite{alistarh2017qsgd}, we can get
%$$\mathbb{E}\|\mathcal{C}_{QSGD}(\mathbf{a})-\mathbf{a}\|^2\le\min(\frac{d}{Z^2},\frac{\sqrt{d}}{Z})\|\mathbf{a}\|^2.$$
%According to \cite{alistarh2017qsgd} and Assumption~\ref{Assump:BoundG}, 

% \begin{eqnarray}
% \label{EQ:QSGDJt}
% J^i_t&=&\mathbb{E}\Big\|\widetilde{\mathbf{U}}^i_t-\mathbf{U}^i_t\Big\|^2\ \overset{(1)}{\le}\  \min(\frac{d}{Z_t^2},\frac{\sqrt{d}}{Z_t})\|\mathbf{U}^i_t\|^2\notag\\
% &=& \min(\frac{d}{Z_t^2},\frac{\sqrt{d}}{Z_t})\Big\|\sum_{j=0}^{E-1} \nabla F_i(\mathbf{w}^i_{t,j},\mathcal{B}^i_{t,j})\Big\|^2\notag\\
% &\le& \min(\frac{dE^2G^2}{Z_t^2},\frac{\sqrt{d}E^2G^2}{Z_t})\notag\\
% &=&\min(\frac{dE^2G^2}{2^{2x_t}},\frac{\sqrt{d}E^2G^2}{2^{x_t}}),
% \end{eqnarray}
\begin{equation}
\label{EQ:QSGDJt}
    \begin{aligned}
    J^i_t&=\mathbb{E}\Big\|\widetilde{\mathbf{U}}^i_t-\mathbf{U}^i_t\Big\|^2\ \overset{(1)}{\le}\  \min(\frac{d}{Z_t^2},\frac{\sqrt{d}}{Z_t})\|\mathbf{U}^i_t\|^2\\
    &= \min(\frac{d}{Z_t^2},\frac{\sqrt{d}}{Z_t})\Big\|\sum_{j=0}^{E-1} \nabla F_i(\mathbf{w}^i_{t,j},\mathcal{B}^i_{t,j})\Big\|^2\\
    &\le \min(\frac{dE^2G^2}{Z_t^2},\frac{\sqrt{d}E^2G^2}{Z_t})=\min(\frac{dE^2G^2}{2^{2x_t}},\frac{\sqrt{d}E^2G^2}{2^{x_t}}),
    \end{aligned}
\end{equation}
where inequality $(1)$ is from Lemma 3.1 in \cite{alistarh2017qsgd}.
% Therefore, the corresponding optimization problem for the QSGD algorithm is
% \begin{eqnarray}
% 	&&\min \sum_{t=1}^{T} \min(\frac{d\eta_t^2}{2^{2x_t}}, \frac{\sqrt{d}\eta_t^2}{2^{x_t}}) \notag\\
% 	&&s.t. \quad x_{1}+x_{2}+\dots+x_{T} \le C\notag\\
% 	&& \qquad x_{t}\ge1\notag\\
% 	&& \qquad x_{t}\le32.
% \end{eqnarray}
Again, Eq.~\eqref{EQ:QSGDJt} is convex with respect to $x_t$. By substituting Eq.~\eqref{EQ:QSGDJt} into the problem $\mathbb{P}2$, we can optimally determine $x_t$ and $Z_t$. 

\section{Performance Evaluation} \label{Experiment}
We have conducted extensive experiments to evaluate the effectiveness in terms of model accuracy, communication traffic and communication time by adaptive compression rates. In this section, we report our key findings with two representative datasets, namely, MNIST and CIFAR-10 \cite{krizhevsky2009learning}, for convex loss and non-convex loss, respectively.

\begin{enumerate}
    \item {\bf Convex Loss Case - MNIST} \cite{sattler2019robust, yang2021achieving}: The MNIST dataset contains 70,000 hand-written digital images with labels 0-9. Each sample is a 28×28 grayscale image. We randomly select 60,000 images as the training set and use the rest 10,000 images as the test set. We train a logistic regression model (which is convex) to classify the MNIST datasets. The model contains a fully connected layer with 784 inputs and 10 outputs (without activation function), and therefore involves  7,850 parameters in total. 
    %{\color{red} introduce your model details here. For example, how many parameters}  
    Similar to \cite{li2019convergence}, we set the learning rate as $\frac{0.01}{1+t*E}$,  where $t$ is the global iteration index.
    %XXX reason for this XXXX 
    
    \item {\bf Non-Convex Loss Case - CIFAR-10}: The CIFAR-10 dataset consists of 60,000 32×32 color images, which belong to 10 classes. We randomly select 50,000 images as the training set and 10,000 images as the test set. We train a convolutional neural network (CNN) model (which is non-convex). The CNN model has  three 3×3 convolutional layers. For the first two layers, each layer is followed by a max pooling layer; there are two fully connected layers for the output of 10 classes of probabilities. A similar model structure has been used in  \cite{wang2020optimizing,mcmahan2017communication} with 122,570 parameters in total. 
    %{\color{red} better to have references if your LR or CNN model has the same structure with other works}
    We set the learning rate as  $\frac{0.05}{1+\sqrt{t*E}/40}$ according to Theorem~\ref{THE:NonConvexCR2}.
    %XXX reason for this? XXXX
\end{enumerate}

In our experiments, we set up 100 clients and a single PS in the FL system and implement the FedAvg algorithm in \cite{mcmahan2017communication} for model training. In each round of communication, the PS will randomly select 10 participating clients  and each client is selected with the probability proportional to the number of samples owned by the client. 
%{\color{blue} what is weight? Is it the proportion of the number of samples? } 
Selected clients will download the latest model parameters from the PS and conduct $E=5$ local iterations.  For local iterations, we use the batch gradient descent algorithm \cite{li2014efficient}, and the batch size is set as 50 and 8 for convex and non-convex loss functions, respectively. These default parameter settings resemble those used in previous study \cite{mcmahan2017communication}.
The total numbers of global iterations for training convex and non-convex loss functions are 200 and 400, respectively. 

We have implemented both the Adaptive PQ and the Adaptive QSGD algorithms  to compress model updates based on our case discussion in the previous section. The standard PQ and QSGD therefore serve as the baselines for comparison. 
FedAvg (without compression) is implemented to evaluate how the model accuracy will be impacted by compression errors. 
We have also implemented an Adaptive TopK compression algorithm to evaluate the robustness of our framework. It is derived from the well-known TopK biased compression algorithm \cite{stich2018sparsified}, which only selects the $K$  model updates that are furthest away from $0$ for transmission. 
%Our adaptive algorithm however XXXX

We evaluate compression algorithms from three critical perspectives, namely, {\em model accuracy}, {\em communication traffic}, and {\em communication time}.

%to represent the adaptive change in the number of centroids of each communication round in training of PQ and QSGD algorithm respectively. It should be noted that the result of the optimization problem may be decimal, so we need to perform a round operation on it. In this experiment, when the sum of the number of centroids after rounding reaches the sum of the number of fixed centroids, the training of the model will be automatically stopped.

\subsection{Comparison of Model Accuracy}

We start from the accuracy results with MNIST, \emph{i.e.}, the convex loss case. We adopt both IID and non-IID sample distributions for our experiments: (1) IID distribution, which randomly allocates 600 samples from the training set to each client; and (2)  Non-IID distribution, which  randomly allocates 600 samples out of 5 classes from the training set  to each client.

For PQ and QSGD, we set the default $Z$, \emph{i.e.}, the number of centroids, to be 16 for compression in each global iteration. Given total 200 global iterations, for both Adaptive PQ and Adaptive QSGD, we have the total traffic limitation $C=\sum_{t=0}^{T=199} d\times \log_2Z_t = 200d\times  \log_2{16}=800d$ bits to ensure a fair comparison for different compression algorithms.

The results of our experiments are shown in Figs.~\ref{PQ_MNIST} for PQ and \ref{QSGD_MNIST} for QSGD. The x-axis represents the number of global iterations, while the y-axis represents the model accuracy on the test dataset. From the results in Figs.~\ref{PQ_MNIST} and \ref{QSGD_MNIST}, one can observe that adaptive algorithms achieve higher model accuracy in all experimental cases, which also shed light on the effectiveness of our framework.

Note that, although Adaptive PQ and Adaptive QSGD can reach over 87\% model accuracy, it is a little bit less than the accuracy of FedAvg without compression because the model compression operation inevitably lowers the model accuracy. The traffic however is reduced significantly with compression, around $32/\log_2{16}=8$ times less than the original, assuming each original model update  takes $4$ bytes. 

%that by optimizing the numbers of centroid in each communication round, the performance of model can be optimized without aggravating the communication.
\begin{figure}[h]
    \centering
    \includegraphics[width=0.9\linewidth]{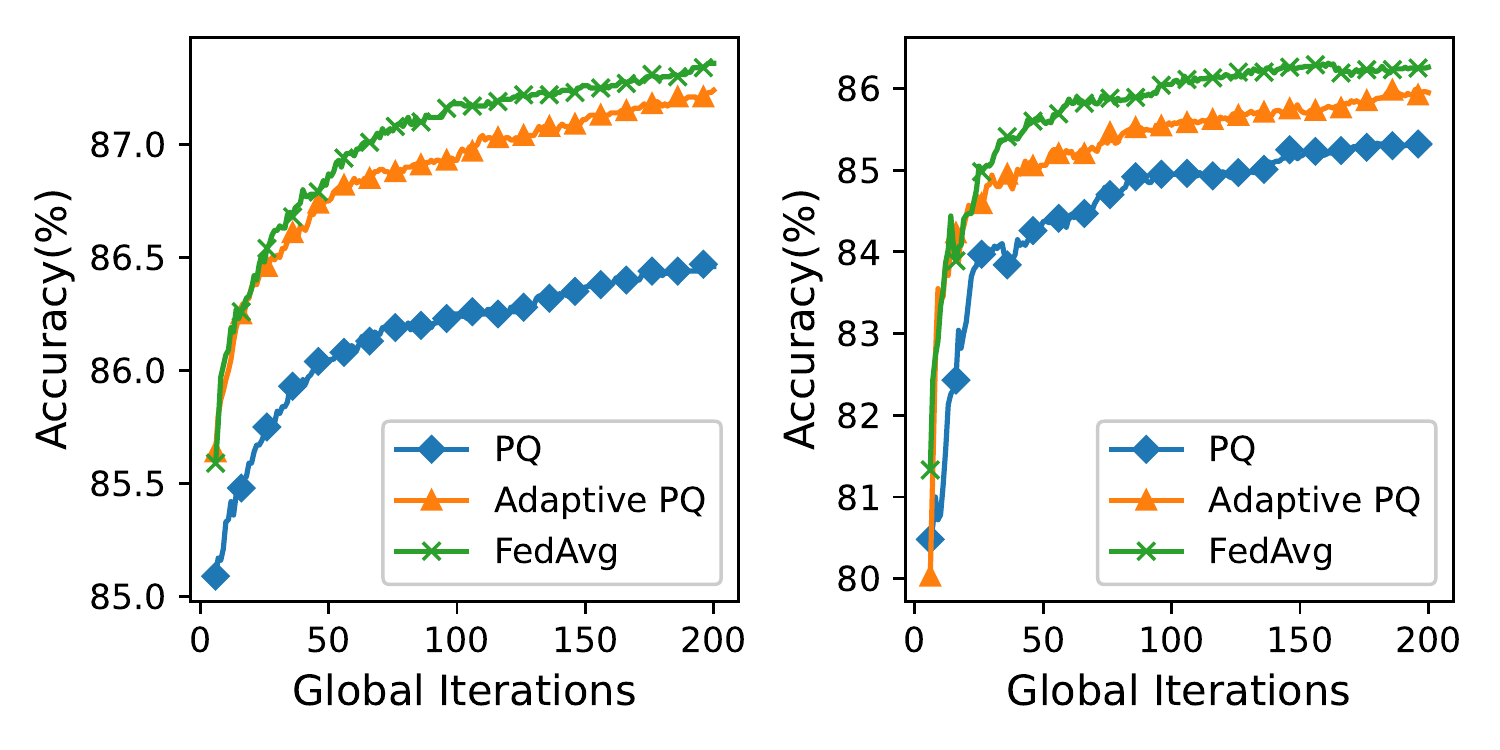}
    \caption{Accuracy comparison of Adaptive PQ and PQ in MNIST with IID (left) and non-IID (right) sample distributions (Our plotting starts from the $5^{th}$ global iteration). }
    \label{PQ_MNIST}
\end{figure}

\begin{figure}[h]
    \centering
    \includegraphics[width=0.9\linewidth]{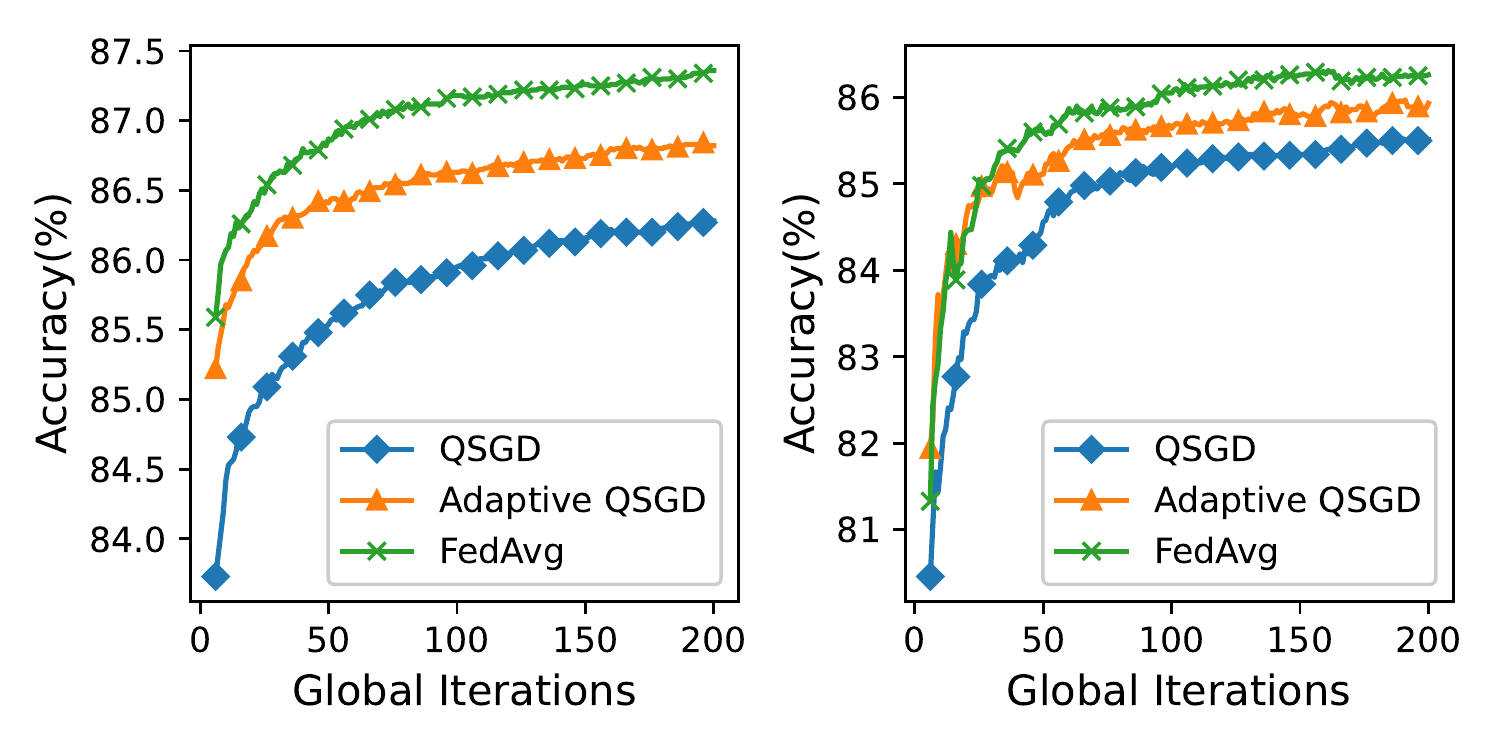}
    \caption{Accuracy comparison of Adaptive QSGD and QSGD in MNIST with IID (left) and non-IID (right) sample distributions (Our plotting starts from the $5^{th}$ global iteration).}
    \label{QSGD_MNIST}
\end{figure}

We have conducted similar experiments with the CIFAR-10 dataset, \emph{i.e.}, non-convex case. Again, we implement two sample distributions: (1) IID distribution, which randomly allocates 500 samples from the training set to each client; and (2)  Non-IID distribution, which  randomly allocates 500 samples out of 5 classes from the training set  to each client.

%\begin{itemize}
 %   \item IID distribution: Each client randomly selects 500 samples from the training set as its local dataset.
  %  \item Non-IID distribution: Each client randomly selects 500 samples out of 5 classes from the training set as its local dataset. 
%\end{itemize}
For both PQ and QSGD, we set $Z= 128$. For Adaptive PQ and Adaptive QSGD, it implies that $C=\sum_{t=0}^{T=399} d\times \log_2{Z_t} = 400d\times \log_2{128} = 2800d$ bits with 400 global iterations. 

The experiment results are presented in Figs.~\ref{PQ_CIFAR} for PQ and \ref{QSGD_CIFAR} for QSGD. From the experimental results, we can observe that Adaptive PQ and Adaptive QSGD can achieve higher model accuracy even if the loss functions are  non-convex. Compared to the results of MNIST, the model accuracy in Figs.~\ref{PQ_CIFAR} and \ref{QSGD_CIFAR} fluctuates more frequently because the classification task of the CIFAR-10 dataset is more complicated. Nevertheless, the model accuracy of our adaptive algorithms  is only slightly lower than that of FedAvg without compression, and adapting compression rates always achieves higher model accuracy than fixed rate compression algorithms, demonstrating the superiority of our framework. 

%XXX are you really saying "improve model accuracy" ? Compression will reduce accuracy compared to baseline. It is not clear what you're comparing to xXXX

%on is better on the whole.
\begin{figure}[h]
    \centering
    \includegraphics[width=0.9\linewidth]{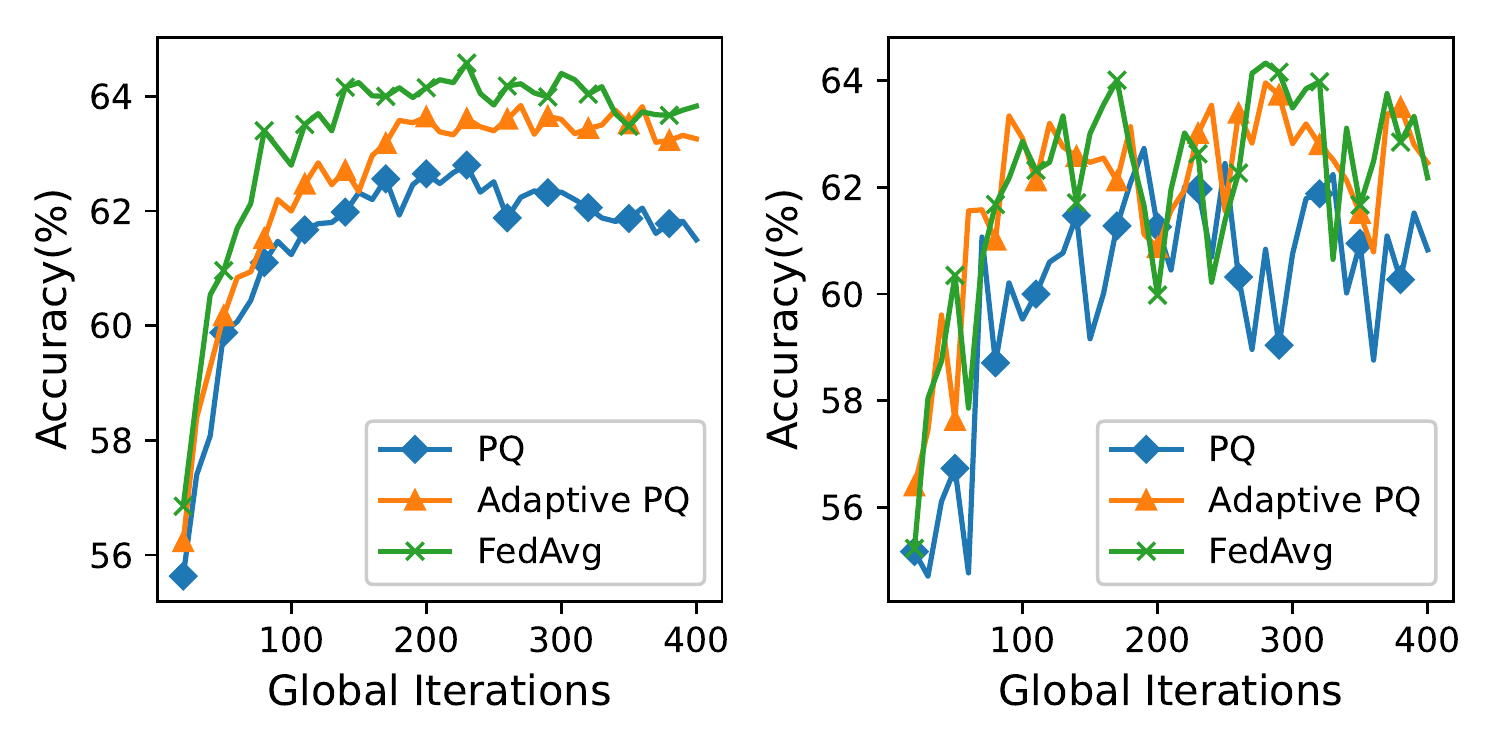}
    \caption{Accuracy comparison of Adaptive PQ and PQ in CIFAR-10 with IID (left) and non-IID (right) sample distributions (Our plotting starts from the $20^{th}$  global iteration).}
    \label{PQ_CIFAR}
\end{figure}

\begin{figure}[h]
    \centering
    \includegraphics[width=0.9\linewidth]{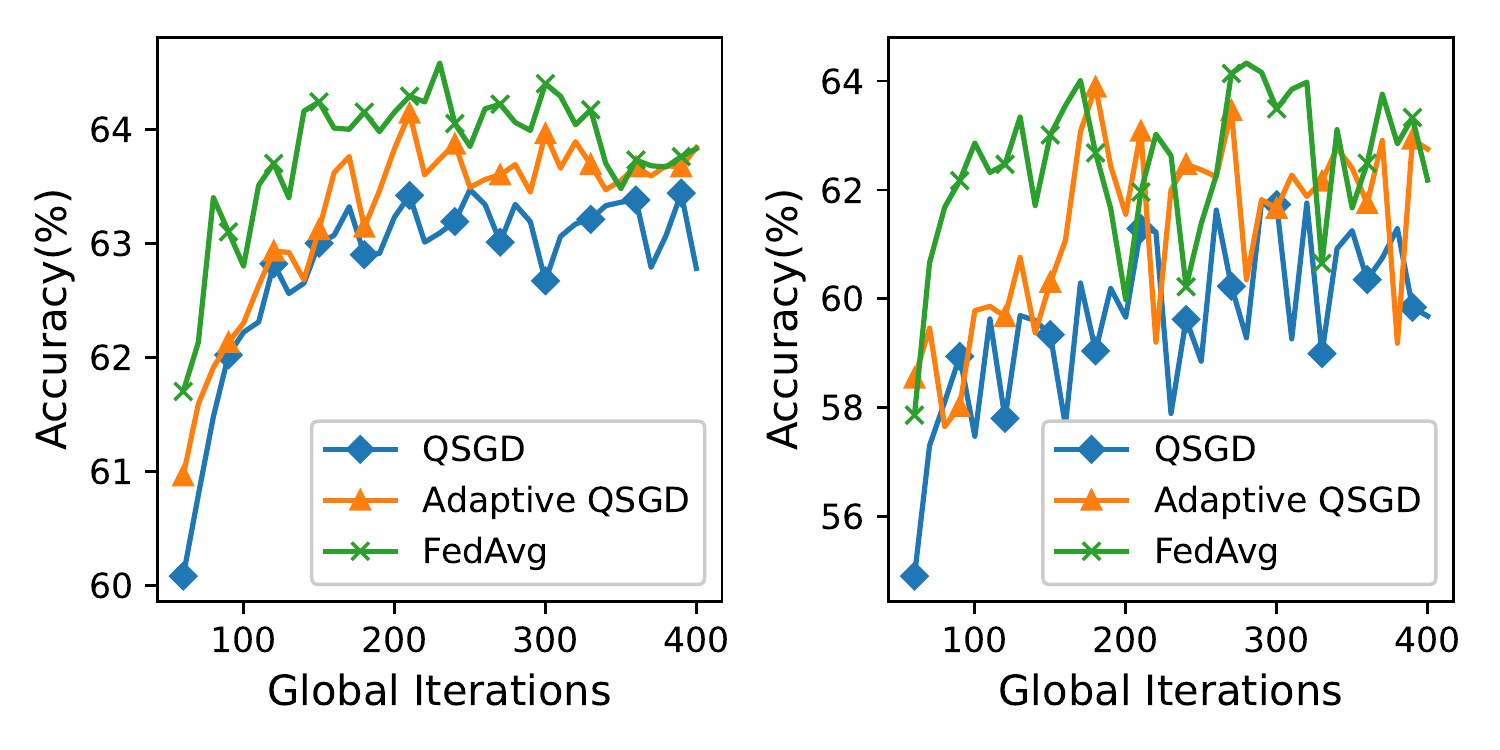}
    \caption{Accuracy comparison of Adaptive QSGD and QSGD in CIFAR-10 with IID (left) and non-IID (right) sample distributions (Our plotting starts from the $60^{th}$ global iteration).}
    \label{QSGD_CIFAR}
\end{figure}

\subsection{Evaluating Robustness}

Although our theoretical analysis has focused on unbiased compression, we believe that our framework is also applicable to biased compression because, in this case, the influence to the final model accuracy is related with the learning rate as well.

%\cite{lin2017deep}. Therefore, it is essential to consider whether our framework is valid on biased algorithms and we conduct experiments on $Top-K$ algorithms with non-IID data distribution to verify it. 

In our Adaptive TopK algorithm, let $K_t$ denote the number of selected model updates in round $t$, then the error function of TopK is  $J^i_t=\mathbb{E}\|\widetilde{\mathbf{U}}^i_t-\mathbf{U}^i_t\|\le(1-K_t/d)\mathbf{U}^i_t$ \cite{stich2018sparsified}. By  substituting it into $\mathbb{P}2$, we can determine the $K_t$ (\emph{i.e.}, $Z_t$) for Adaptive TopK.
We compare the model accuracy between Adaptive TopK and TopK with MNIST and CIFAR-10. The sample distributions are non-IID, same as the non-IID distributions  in previous experiments.

To restrict the total communication traffic, we set $K=3\%d$ for MNIST and $K= 1\%d$ for CIFAR-10 for the TopK algorithm.  In our experiments, each model update consumes 32 bits. The traffic limit $C$ for Adaptive TopK is $\sum_{t=0}^{T=199}{K_t}\times( 32+\log_2d) =6d\times( 32+\log_2d)$ bits for MNIST where $\log_2d$ bits are used to indicate the index of each transmitted model update. Similarly, we have $C = \sum_{t=0}^{T=399}K_t\times( 32+\log_2d) = 4d\times( 32+\log_2d)$ bits for CIFAR-10.
%, assuming  each original model update takes 32 bits.

The experimental results are presented in Fig.~\ref{TopK} showing  that Adaptive TopK significantly outperforms TopK in terms of model accuracy.  Note that the model accuracy of TopK and Adaptive TopK is lower than that of FedAvg, because the compression rate of TopK is very high. This will be further verified in the next experiment on traffic consumption. 

%the framework is also effective for biased algorithms. Moreover, the optimized effect achieved by the framework in biased algorithms is stronger than that in unbiased algorithms. So we can infer that compression error caused by biased algorithms has a greater weight on the model convergence. We will explore the specific impact of biased algorithms on model convergence in future work.

%{\color{blue} what is the restriction of total communication traffic $C$?}

\begin{figure}[h]
    \centering
    \includegraphics[width=0.9\linewidth]{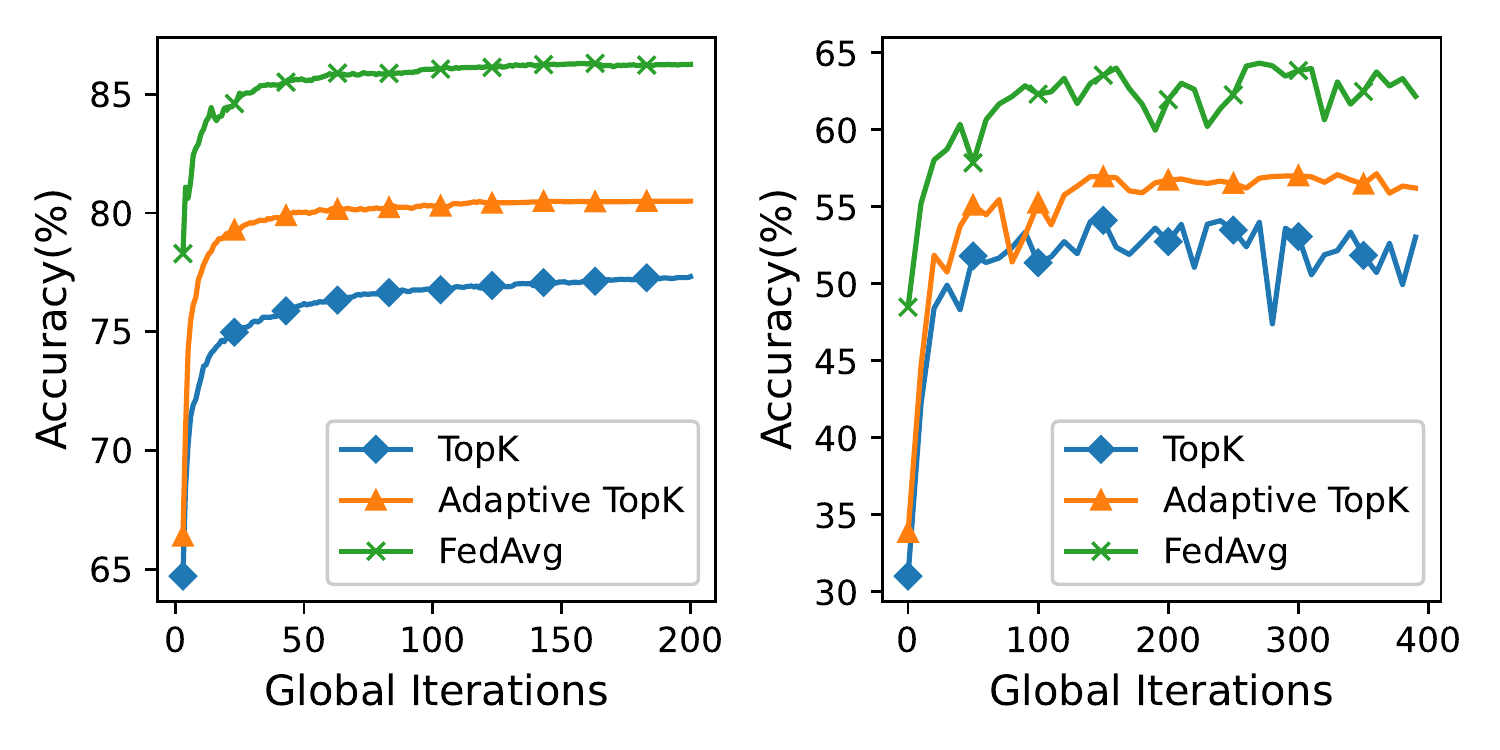}
    \caption{Accuracy comparison of Adaptive TopK and TopK in MNIST (left) and CIFAR-10 (right) with non-IID sample distributions.}
    \label{TopK}
\end{figure}

\subsection{Comparison of Communication Traffic and Time}

We have implemented a network simulator \cite{nishio2019client}. 

To evaluate the consumed total communication traffic and the communication time towards the target accuracy, we simulated wireless communications based on LTE networks, and use the urban channel model defined in the ITU-R M.2135-1 for Micro NLOS communication \cite{series2009guidelines}. In the network, the antenna heights of the PS and clients are 11m and 1m, and the antenna gains are 20dBm and 0dBi, respectively. The carrier frequency is 2.5GHz and 10 resource blocks are allocated to a client in each 0.5ms time slot. Under these settings, the average uplink 
%{\color{blue} uplink or downlink throughput? } 
throughput of clients is 1.4 Mbit/s. To emulate the dynamics in real networks, the network speed is sampled from the Gaussian Distribution and its standard deviation is 10\% of the mean throughput. 
%{\color{blue} what is the role of maximum throughput? }
The communication time in the $t^{th}$ round is estimated as $\max_{k\in \mathcal{K}_t} \frac{M}{\theta_k}$, where $\theta_k$ is the network speed of client $k$ and $M$ is the traffic.

\begin{table}[h]
	\centering
	\caption{Comparison of total uplink communication traffic and communication time from clients to the PS between fixed and adaptive compression algorithms with non-IID distribution}
	\begin{tabular}{|c|c|c|c|c|c|c|}
		\hline
		\multirow{2}{*}{Alg.} & \multirow{2}{*}{DataSet} & \multirow{2}{*}{\tabincell{c}{Target\\Accu.}}  & \multicolumn{2}{c|}{Traffic(MB)} & \multicolumn{2}{c|}{Comm. Time(s)}\\ \cline{4-7} 
		& & & Fixed & Adap. & Fixed & Adap.\\ \hline
		\multirow{2}{*}{PQ} & {MNIST} & 85\% & 5.0 & 1.9 & 3.4 & 1.3  \\ \cline{2-7} 
		& {CIFAR-10} & 62\% & 173.9 & 94.0 & 117.7 & 89.9  \\ \hline
		\multirow{2}{*}{QSGD} & {MNIST} & 85\% & 2.4 & 1.6 & 1.6 & 1.1 \\ \cline{2-7} 
		& {CIFAR-10} & 61.5\% & 255.7 & 175.9 & 172.8 & 156.0 \\ \hline
		\multirow{2}{*}{TopK} & {MNIST} & 77\% & 1.7 & 0.2 & 1.1 & 0.1 \\ \cline{2-7} 
		& {CIFAR-10} & 54.5\% & 21.1 & 10.9 & 14.3 & 7.4 \\ \hline
		\multirow{2}{*}{FedAvg} & {MNIST} & 85\% & \multicolumn{2}{c|}{7.2} & \multicolumn{2}{c|}{5.0} \\ \cline{2-7} 
		& {CIFAR-10} & 62\% & \multicolumn{2}{c|}{374.1} & \multicolumn{2}{c|}{258.2} \\ \hline
	\end{tabular}
	\label{CommunicationTraffic}
\end{table}

We evaluate the consumed total uplink communication traffic from clients to the PS and the total communication time towards the target accuracy. 
Based on the experiments with non-IID sample distributions in Figs.~\ref{PQ_MNIST}-\ref{TopK}, we present the comparison results in 
Table~\ref{CommunicationTraffic}.  %For each experiment case, we cumulate the consumed uplink communication traffic over multiple communication rounds until the model accuracy on the test dataset exceeds the target model accuracy. 
The results in Table~\ref{CommunicationTraffic} shows the tradeoff between model accuracy and the compression rate. FedAvg without compression consumes much more communication traffic and longer training time than other algorithms though its accuracy is slightly higher. 
Our adaptive  algorithms always consume  less  communication traffic and communication time than algorithms with fixed compression rates. The model accuracy of TopK is slightly lower, but it consumes much less communication traffic and communication time than PQ and QSGD. It is worth noting that Adaptive TopK can further reduce the traffic and time cost of Fixed TopK. %reducing uplink communication traffic by 16.7\%-75.3\%. 

% \begin{table}[h]
% 	\centering
% 	\caption{Comparison of total uplink communication traffic from each client to the PS between fixed and adaptive compression algorithms (unit: MB)}
% 	\begin{tabular}{|c|c|c|c|c|c|c|}
% 		\hline
% 		Alg. & DataSet & Dist. & \tabincell{c}{Target\\Accu.}  & \tabincell{c}{Fixed\\Traffic} & \tabincell{c}{Adap.\\Traffic} & \tabincell{c}{Traffic\\Reduce}\\ \hline
% 		\multirow{4}{*}{PQ} & \multirow{2}{*}{MNIST} & IID & 86\% & 7.3 & 1.8 & 75.3\% \\ \cline{3-7} 
% 		& & NonIID & 85\% & 5.0 & 1.7 & 66.0\% \\ \cline{2-7} 
% 		& \multirow{2}{*}{CIFAR-10} & IID & 62.5\% & 163.6 & 72.5 & 55.7\% \\ \cline{3-7} 
% 		& & NonIID & 62\% & 173.9 & 94.0 & 45.9\%  \\ \hline
% 		\multirow{4}{*}{QSGD} & \multirow{2}{*}{MNIST} & IID & 86\% & 4.0 & 1.3 & 67.5\% \\ \cline{3-7} 
% 		& & NonIID & 85\% & 2.4 & 2.0 & 16.7\% \\ \cline{2-7} 
% 		& \multirow{2}{*}{CIFAR-10} & IID & 63.4\% & 204.6 & 165.7 & 19.0\% \\ \cline{3-7} 
% 		& & NonIID & 61.5\% & 255.7 & 175.9 & 31.2\% \\ \hline
% 	\end{tabular}
% 	\label{CommunicationTraffic}
% \end{table}

\section{Conclusion} \label{Conclusion}

Federated learning involves heavy data communication across distributed network nodes. Compression is necessary for FL over today's capacity limited and dynamic Internet and wireless networks. In this paper, we for the first time analyzed the convergence rate of FedAvg with adaptive compression for both strongly convex and non-convex loss functions. We presented an  optimized framework that balances the  compression rates with learning rates, seeking to maximize the overall model accuracy. Our rate adaptation framework is generally applicable with diverse  compression algorithms, and we have closely examine the implementation with PQ and QSGD. Experiment results with such representative datasets as MNIST and CIFAR-10 suggested that our solution effectively reduces network traffic yet maintains high model accuracy in FL. 

Our work is an initial attempt toward this direction and there remain spaces to explore. In particular, with advances in  compression algorithms, it can be difficult to derive an exact expression of $J_t^i$ for a future algorithm. A possible solution is to fit the compression error function with different compression rates. This trial process is lightweight in communication since the PS only needs to communicate parameters of the fitting function with clients. We will investigate its effectiveness in theory as well conduct larger scale experiments to further optimize our solutions. 

%we consider the optimization of FL communication by unbiased quantization algorithm and analyze the convergence of the trained model when the loss function is strongly convex and non-convex. According to the analysis results, we find that the influence of compression error on the convergence of the model is related to the learning rate. So we design an optimization problem to improve the accuracy of the model without aggravating the communication by adaptively adjusting the compression rate at each communication round. PQ and QSGD algorithms were used as specific algorithms to analyze their compression errors, and MNIST and CIFAR-10 datasets were used for verification. The experimental results show that the performance of the final model can be significantly improved by adjusting the compression rate adaptively in each communication round.

\appendices
\section{Proof of Theorem~\ref{THE:ConvexCR1}}
\label{ProofOfTheorem1}
We use $\mathbf{v}^i_{t,j+1}=\mathbf{w}^i_{t,j}-\eta_t\nabla F_i(\mathbf{w}^i_{t,j},\mathcal{B}^i_{t,j})$ to represent the result of a local epoch result on client $i$, whereas $\mathbf{w}^i_{t,j+1}=\mathbf{v}^i_{t,j+1}$ if $j+1 < E$, else, $\mathbf{w}^i_{t+1,0}=\mathbf{w}^i_{t,0}-\frac{\eta_t}{K}\sum_{i\in\mathcal{K}_t}\widetilde{\mathbf{U}}^i_{t}$ and $\mathbf{v}^i_{t+1,0}=\mathbf{v}^i_{t,E}$. We further define $\bar{\mathbf{v}}_{t,j}=\sum_{i=1}^N p_i \mathbf{v}^i_{t,j}$, $\bar{\mathbf{w}}_{t,j}=\sum_{i=1}^N p_i \mathbf{w}^i_{t,j}$, $\bar{\mathbf{v}}_{t}=\bar{\mathbf{v}}_{t,0}$ and $\bar{\mathbf{w}}_{t}=\bar{\mathbf{w}}_{t,0}$. 
%Also, we define $\bar{\mathbf{g}}_{t,j}=\sum_{i=1}^N p_i \nabla F_i(\mathbf{w}^i_{t,j})$ and $\mathbf{g}_{t,j}=\sum_{i=1}^N p_i \nabla F_i(\mathbf{w}^i_{t,j},\mathcal{B}^i_t)$ such that $\bar{\mathbf{v}}_{t,j+1}=\bar{\mathbf{w}}_{t,j}-\eta_t \mathbf{g}_{t,j}$ and $\mathbb{E}[\mathbf{g}_{t,j}]=\bar{\mathbf{g}}_{t,j}$.
%We need to leverage the following lemmas for analysis.%to analyze the convergence of the model.

%From \cite{li2019convergence}, we have the following two lemmas:

\begin{lemma}
	\label{Lemma:v-w}
	%(Results of one step of iteration)
	Let Assumptions \ref{Assump:Smooth}-\ref{Assump:BoundG} hold, $\eta_t\le\frac{1}{4L}$, we have $\mathbb{E}\left\|\bar{\mathbf{v}}_{t,j+1}-\mathbf{w}^*\right\|^2=(1-\eta_t\mu)\mathbb{E}\left\|\bar{\mathbf{w}}_{t,j}-\mathbf{w}^*\right\|^2+\eta_t^2\frac{1}{B}\sum_{i=1}^Np_i^2\sigma^2+6L\eta_t^2\Gamma_c+2\eta_t^2(E-1)^2G^2.$
% 	\begin{equation}
% 	    \begin{aligned}
% 	        &\mathbb{E}\left\|\bar{\mathbf{v}}_{t,j+1}-\mathbf{w}^*\right\|^2=(1-\eta_t\mu)\mathbb{E}\left\|\bar{\mathbf{w}}_{t,j}-\mathbf{w}^*\right\|^2\\
% 		    &\quad +\eta_t^2\frac{1}{B}\sum_{i=1}^Np_i^2\sigma^2+6L\eta_t^2\Gamma_c+2\eta_t^2(E-1)^2G^2.
% 	    \end{aligned}
% 	\end{equation}
% 	\begin{eqnarray}
% 		\label{EQ:Lemma1}
% 		&&\mathbb{E}\left\|\bar{\mathbf{v}}_{t,j+1}-\mathbf{w}^*\right\|^2\notag\\
% 		&&=(1-\eta_t\mu)\mathbb{E}\left\|\bar{\mathbf{w}}_{t,j}-\mathbf{w}^*\right\|^2+\eta_t^2\frac{1}{B}\sum_{i=1}^Np_i^2\sigma^2\notag\\
% 		&&\quad +6L\eta_t^2\Gamma+2\eta_t^2(E-1)^2G^2.
% 	\end{eqnarray}
\end{lemma}
% \begin{lemma}
% 	\label{Lemma:v-w}
% 	(Results of one step of iteration) Let Assumption \ref{Assump:Convex}-\ref{Assump:Smooth} hold and $\eta_t\le\frac{1}{4L}$, we can get
% 	\begin{eqnarray}
% 		\label{EQ:Lemma1}
% 		&&\mathbb{E}\left\|\bar{\mathbf{v}}_{t,j+1}-\mathbf{w}^*\right\|^2\notag\\
% 		&&=(1-\eta_t\mu)\mathbb{E}\left\|\bar{\mathbf{w}}_{t,j}-\mathbf{w}^*\right\|^2+\eta_t^2\mathbb{E}\left\|\mathbf{g}_{t,j}-\bar{\mathbf{g}}_{t,j}\right\|^2\notag\\
% 		&&\quad +6L\eta_t^2\Gamma+2\mathbb{E}\Big[\sum_{i=1}^Np_i\left\|\bar{\mathbf{w}}_{t,j}-\mathbf{w}^i_{t,j}\right\|^2\Big].
% 	\end{eqnarray}
% \end{lemma}

% \begin{lemma}
% 	\label{Lemma:VariOfg}
% 	(Bounding the variance) Let Assumption \ref{Assump:LocalVar} holds. It follows that
% 	\begin{equation}
% 		\label{EQ:Lemma2}
% 		\mathbb{E}\left\|\mathbf{g}_{t,j}-\bar{\mathbf{g}}_{t,j}\right\|^2\le \frac{1}{B}\sum_{i=1}^Np_i^2\sigma^2.
% 	\end{equation}
% \end{lemma}

% \begin{lemma}
% 	\label{Lemma:LocalGandGloG}
% 	(Bounding the divergence of $\{\mathbf{w}^i_t\}$) Let Assumption \ref{Assump:BoundG} holds, we have
% 	\begin{equation}
% 		\label{EQ:Lemma3}
% 		\mathbb{E}\Big[\sum_{i=1}^Np_i\left\|\bar{\mathbf{w}}_{t,j}-\mathbf{w}^i_{t,j}\right\|^2\Big]\le \eta_t^2(E-1)^2G^2.
% 	\end{equation}
% \end{lemma}

\begin{lemma}
	\label{Lemma:Unbias}
	%(Unbiased sampling scheme and  compression)
	Let Assumptions \ref{Assump:Smooth}-\ref{Assump:BoundG} hold, we have $\mathbb{E}[\bar{\mathbf{w}}_{t,j}]=\bar{\mathbf{v}}_{t,j}.$
% 	\begin{equation}
% 		\label{EQ:UnbiasedHK}
% 		\mathbb{E}[\bar{\mathbf{w}}_{t,j}]=\bar{\mathbf{v}}_{t,j}.
% 	\end{equation}
\end{lemma}

Lemma.~\ref{Lemma:v-w}-\ref{Lemma:Unbias} has been proved in \cite{li2019convergence}.
%We then bound the variance of $\bar{\mathbf{w}}_t$.

% \begin{lemma}
% 	\label{Lemma:PartialClientsComp}
% 	%(Bounding variance of $\bar{\mathbf{w}}_t$) 
% 	Let Assumptions \ref{Assump:Smooth}-\ref{Assump:BoundG} hold,  we have $$\mathbb{E}\left\|\bar{\mathbf{v}}_{t+1}-\bar{\mathbf{w}}_{t+1}\right\|^2\le \frac{1}{K}(\sum_{i=1}^Np_{i} \eta_t^2J^{i}_{t} +  \eta_{t}^2E^2G^2).$$
% \end{lemma}

%\begin{proof}
    %We first calculate the upper bound of the variance of $\bar{\mathbf{w}}_{t+1}$.
    When Assumptions \ref{Assump:Smooth}-\ref{Assump:BoundG} hold, we will analyze the variance between $\bar{\mathbf{v}}_{t+1}$ and $\bar{\mathbf{w}}_{t+1}$ as follow. %proof lemma 3 as follow.
    \begin{eqnarray}
    \label{VarianceOfWAndVPartial}
        &&\mathbb{E}\left\|\bar{\mathbf{w}}_{t+1}-\bar{\mathbf{v}}_{t+1}\right\|^2=\mathbb{E}\Big\|\frac{\eta_t}{K}\sum_{i\in\mathcal{K}_{t}}\widetilde{\mathbf{U}}^i_{t} -\eta_t\sum_{i=1}^Np_i\mathbf{U}^i_{t}\Big\|^2\notag\\
        &&=\frac{\eta_t^2}{K^2}\mathbb{E}\big[\big\|\sum_{l=1}^K\big( \widetilde{\mathbf{U}}^{i_l}_{t} - \sum_{i=1}^Np_i\mathbf{U}^i_{t}\big)\big\|^2\big]\notag\\
        &&=\frac{\eta_t^2}{K^2}\sum_{l=1}^K\mathbb{E}\Big[\Big\| \widetilde{\mathbf{U}}^{i_l}_{t} - \sum_{i=1}^Np_i\mathbf{U}^i_{t}\Big\|^2\Big]\notag\\
        %&&=\frac{\eta_t^2}{K^2}\sum_{l=1}^K\mathbb{E}\Big[
        %\sum_{i^{'}=1}^Np_{i^{'}}
        %\Big\| \widetilde{\mathbf{U}}^{i^{'}}_{t} - \sum_{i=1}^Np_i\mathbf{U}^i_{t}\Big\|^2
        %\Big]\notag\\
        &&=\frac{\eta_t^2}{K}\sum_{i^{'}=1}^Np_{i^{'}}\mathbb{E}\Big\| \widetilde{\mathbf{U}}^{i^{'}}_{t} - \sum_{i=1}^Np_i\mathbf{U}^i_{t}\Big\|^2\notag\\
        &&=\frac{\eta_t^2}{K}\sum_{i^{'}=1}^Np_{i^{'}}\mathbb{E}\Big[\Big\|\widetilde{\mathbf{U}}^{i^{'}}_{t}-\mathbf{U}^{i^{'}}_{t}\Big\|^2 + \Big\|\sum_{i=1}^Np_i\mathbf{U}^i_{t}-\mathbf{U}^{i^{'}}_{t}\Big\|^2\Big]\notag\\
        &&\le\frac{1}{K}\sum_{i^{'}=1}^Np_{i^{'}}\eta_t^2J^{i'}_t+\frac{\eta_t^2}{K}\mathbb{E}\Big[\sum_{i^{'}=1}^Np_{i^{'}}\left\|\mathbf{U}^{i^{'}}_{t+1}\right\|^2\Big]\notag\\
        &&\le\frac{1}{K}\sum_{i^{'}=1}^Np_{i^{'}}\eta_t^2J^{i'}_t+\frac{\eta_t^2G^2E^2}{K}\notag.
    \end{eqnarray}
%\end{proof}

%\subsection{Convergence Rate with Convex Loss} 
We will use the lemmas to derive the convergence of the model. We first analyze the interval between $\bar{\mathbf{w}}_{t+1}$ and $\mathbf{w}^*$.
\begin{equation}
    \label{PartialClientGap}
    \begin{split}
        &\mathbb{E}\left\|\bar{\mathbf{w}}_{t+1} - \mathbf{w}^*\right\|^2=\mathbb{E}\left\|\bar{\mathbf{w}}_{t+1} - \bar{\mathbf{v}}_{t+1} + \bar{\mathbf{v}}_{t+1} - \mathbf{w}^*\right\|^2\\
        %&=\mathbb{E}\left\|\bar{\mathbf{w}}_{t+1} - \bar{\mathbf{v}}_{t+1}\right\|^2 + \mathbb{E}\left\|\bar{\mathbf{v}}_{t+1} - \mathbf{w}^*\right\|^2 + \\
        %&\qquad2\mathbb{E}<\bar{\mathbf{w}}_{t+1} - \bar{\mathbf{v}}_{t+1}, \bar{\mathbf{v}}_{t+1} - \mathbf{w}^*>\\
        &\le (1-\eta_t\mu)\left\|\bar{\mathbf{w}}_{t,E-1}-\mathbf{w}^*\right\|^2+\eta_t^2\tau+\eta_t^2\frac{E^2G^2}{K}+\psi_{t},\notag
        %&\le a\left\|\bar{\mathbf{w}}_{t,E-1}-\mathbf{w}^*\right\|^2+\eta_t^2\tau+\eta_t^2\frac{E^2G^2}{K}+\psi_{t},
    \end{split}
\end{equation}
where 
%$1-\eta_t\mu<a<1$ for all $t$, 
$\tau=\frac{\sum_{i=1}^Np_i^2\sigma^2}{B}+6L\Gamma_c+2(E-1)^2G^2$ and $\psi_{t}= \frac{1}{K}\sum_{i=1}^Np_{i}\eta_t^2 J^{i}_{t}$.
According to Lemma~\ref{Lemma:v-w} and $\bar{\mathbf{w}}^i_{t,j}=\bar{\mathbf{v}}^i_{t,j}$, we can get $\mathbb{E}\left\|\bar{\mathbf{w}}_{t,j+1}-\mathbf{w}^*\right\|^2\le (1-\eta_t\mu)\left\|\bar{\mathbf{w}}_{t,j}-\mathbf{w}^*\right\|^2+\eta_t^2\tau$, where $j < E - 1$. 
% So we can get
% \begin{eqnarray}
%     &&\mathbb{E}\left\|\bar{\mathbf{w}}_{t+1} - \mathbf{w}^*\right\|^2\notag\\
%     &&\le (1-\eta_t\mu)^E\mathbb{E}\left\|\bar{\mathbf{w}}_{t} - \mathbf{w}^*\right\|^2 + \eta_t^2E\tau + \eta_t^2\frac{E^2G^2}{K} + \psi_{t}.
% \end{eqnarray}
Through accumulation, we can obtain
% \begin{eqnarray}
%     &&\mathbb{E}\left\|\bar{\mathbf{w}}_{T} - \mathbf{w}^*\right\|^2\\
%     &&\le\sum_{t=0}^{T-1}\eta_t^2\left(\frac{E^2G^2}{K} + \frac{\sum_{i=1}^Np_{i} J^{i}_{t}}{K}\right) + \sum_{t=0}^{T-1}\eta_t^2\tau\frac{1-(1-\eta_t\mu)^E}{\eta_t\mu}\notag\\% \eta_t^2\tau\frac{1-a^{T*E}}{1-a} \notag\\
%     &&\qquad +\left(\prod_{t=0}^{T-1}(1-\eta_t\mu)\right)^E\left\|\bar{\mathbf{w}}_{0} - \mathbf{w}^*\right\|^2\notag
% \end{eqnarray}
\begin{eqnarray}
    \mathbb{E}\left\|\bar{\mathbf{w}}_{t+1} - \mathbf{w}^*\right\|^2\le (1-\eta_t\mu)\mathbb{E}\left\|\bar{\mathbf{w}}_{t} - \mathbf{w}^*\right\|^2+ \eta_t^2\alpha+\psi_t\notag,
\end{eqnarray}
where $\alpha=\frac{E\sum_{i=1}^Np_i^2\sigma^2}{B}+6EL\Gamma_c+2E(E-1)^2G^2+\frac{E^2G^2}{K}$.

For a decreasing learning rate $\eta_t=\frac{\beta}{t+\gamma}$ where $\beta>\frac{1}{\mu}$ and $\gamma>0$, we now prove via induction that $\mathbb{E}\left\|\bar{\mathbf{w}}_{t} - \mathbf{w}^*\right\|^2\le\frac{\sum_{i=0}^{t-1}[(i+1+\gamma)\psi_i] + v}{\gamma+t}$ where $v=\max\left\{\frac{\beta^2\alpha}{\beta\mu-1},\gamma\left\|\bar{\mathbf{w}}_{0} - \mathbf{w}^*\right\|^2\right\}$. 

Note that the inequality above holds when $t=0$ and we assume that there is $t$ that makes the inequality valid,
\begin{eqnarray}
    &&\mathbb{E}\left\|\bar{\mathbf{w}}_{t+1} - \mathbf{w}^*\right\|^2\le(1-\eta_t\mu)\mathbb{E}\left\|\bar{\mathbf{w}}_{t} - \mathbf{w}^*\right\|^2+\eta_t^2\alpha+\psi_{t}\notag\\
    &&\leq (1-\frac{\beta\mu}{t+\gamma})\frac{\sum_{i=0}^{t-1}[(i+1+\gamma)\psi_i] + v}{\gamma+t}+\frac{\beta^2\alpha}{(t+\gamma)^2}+\psi_{t}\notag\\
    &&\le\frac{(t+\gamma-1)}{(t+\gamma)^2}v+ \Big[\frac{\beta^2\alpha}{(t+\gamma)^2}-\frac{(\beta\mu-1)}{(t+\gamma)^2}v\Big]\notag\\
    &&\qquad +\frac{t+\gamma-\beta\mu}{t+\gamma}\frac{\sum_{i=0}^{t-1}((i+1+\gamma)\psi_i)}{t+\gamma}+\psi_{t}\notag\\
    &&\le\frac{1}{t+\gamma+1}v+\frac{t+\gamma-1}{(t+\gamma)^2}\sum_{i=0}^{t-1}\big((i+1+\gamma)\psi_i\big)+\psi_{t}\notag\\
    &&\le\frac{v_{t+1}}{\gamma+t+1}.
\end{eqnarray}
% \begin{eqnarray}
%     %\begin{split}
%         &&\mathbb{E}\left\|\bar{\mathbf{w}}_{t+1} - \mathbf{w}^*\right\|^2\notag\\
%         &&\le(1-\eta_t\mu)\mathbb{E}\left\|\bar{\mathbf{w}}_{t} - \mathbf{w}^*\right\|^2+\eta_t^2\alpha+\psi_{t+1}\notag\\
%         &&\leq (1-\frac{\beta\mu}{t+\gamma})\frac{v_t}{\gamma+t}+\frac{\beta^2\alpha}{(t+\gamma)^2}+\psi_{t+1}\notag\\
%         &&\le (1-\frac{\beta\mu}{t+\gamma})\Big(\frac{1}{t+\gamma}\frac{\beta^2\alpha}{\beta\mu-1}+\frac{\sum_{i=0}^t((i+\gamma)\psi_i)}{t+\gamma}\Big)\notag\\
%         &&\qquad +\frac{\beta^2\alpha}{(t+\gamma)^2}+\psi_{t+1}\notag\\
%         &&\le\frac{(t+\gamma-1)}{(t+\gamma)^2}\frac{\beta^2\alpha}{\beta\mu-1}+ \Big[\frac{\beta^2\alpha}{(t+\gamma)^2}-\frac{(\beta\mu-1)}{(t+\gamma)^2}\frac{\beta^2\alpha}{\beta\mu-1}\Big]\notag\\
%         &&\qquad +\frac{t+\gamma-\beta\mu}{t+\gamma}\frac{\sum_{i=0}^t((i+\gamma)\psi_i)}{t+\gamma}+\psi_{t+1}\notag\\
%         &&\le\frac{1}{t+\gamma+1}\frac{\beta^2\alpha}{\beta\mu-1}+\frac{t+\gamma-1}{(t+\gamma)^2}\sum_{i=0}^t((i+\gamma)\psi_i)+\psi_{t+1}\notag\\
%         %&&\le\frac{1}{t+\gamma+1}\frac{\beta^2\tau}{\beta\mu-1}+\notag\\
%         %&&\qquad \frac{1}{t+\gamma+1}(\sum_{i=1}^t((i+\gamma)\psi_i)+(t+\gamma+1)\psi_{t+1})\notag\\
%         &&\le\frac{v_{t+1}}{\gamma+t+1}.
%     %\end{split}
% \end{eqnarray}
Let $\beta=\frac{2}{\mu}$, $\gamma=\frac{8L}{\mu}$ and $\kappa=\frac{L}{\mu}$,  we have

\begin{eqnarray}
    %\begin{split}
        &&\mathbb{E}\left\|\bar{\mathbf{w}}_{T} - \mathbf{w}^*\right\|^2\le\frac{\sum_{i=0}^{T-1}[(i+1+\gamma)\psi_i] + v}{\gamma+T}\notag\\
        &&\le \frac{\frac{\beta^2\alpha}{\beta\mu-1}+\sum_{i=0}^{T-1}[(i+1+\gamma)\psi_i]+\gamma\left\|\bar{\mathbf{w}}_{0} - \mathbf{w}^*\right\|^2}{\gamma+T}\notag\\
        &&=\frac{\frac{4\alpha}{\mu}+\sum_{t=0}^{T-1}\frac{4\sum_{i=1}^Np_iJ^i_t}{(t+\gamma)K\mu^2}+8\kappa\left\|\bar{\mathbf{w}}_{0} - \mathbf{w}^*\right\|^2}{\gamma+T}.
    %\end{split}
\end{eqnarray}
%So the theorem 1 is proved.

\section{Proof of Theorem~\ref{THE:NonConvexCR2}}
\label{ProofOfTheorem2}

The following two lemmas adapted from \cite{yang2021achieving} will be used to derive the convergence rate with non-convex loss.

\begin{lemma}
    \label{Lemma:LocalBound}
    Given assumptions \ref{Assump:Smooth} and \ref{Assump:LocalVar} and learning rate $\eta_t\le\frac{1}{8LE}$, we have %$\mathbb{E}\|\mathbf{w}^i_{t,j}-\mathbf{w}_t\|^2=
    $\mathbb{E}\Big\|\sum_{k=0}^{j-1}\nabla F_i(\mathbf{w}^i_{t,k},\mathcal{B}^i_{t,k})\Big\|^2\le5E\eta_t^2(\sigma^2+6E\Gamma_n)+30E^2\eta_t^2\|\nabla F(\mathbf{w}_t)\|^2.$
%     \begin{eqnarray}
% 		&&\mathbb{E}\|\mathbf{w}^i_{t,j}-\mathbf{w}_t\|^2=\mathbb{E}\Big\|\sum_{k=0}^{j-1}\nabla F_i(\mathbf{w}^i_{t,k},\mathcal{B}^i_{t,k})\Big\|^2\notag\\
% 		&&\le5E\eta_t^2(\sigma^2+6E\Gamma_n)+30E^2\eta_t^2\|\nabla F(\mathbf{w}_t)\|^2.
% 	\end{eqnarray}
\end{lemma}
\begin{lemma}
    %Given assumptions \ref{Assump:SelectScheme}, 
    After defining $\mathbf{t}^i_t=\sum_{j=0}^{E-1}\nabla F_i(\mathbf{w}^i_{t,j})$, we have
        $$\mathbb{E}\|\sum_{i\in \mathcal{K}_{t}} \mathbf{t}^i_t\|^2=
        K\sum_{i=1}^Np_i\mathbb{E}\left\|\mathbf{t}^i_t\right\|^2+K(K-1) \mathbb{E}\|\sum_{i=1}^Np_i\mathbf{t}^i_t\|^2.$$
\end{lemma}

%\subsection{Convergence Rate with Non-convex Loss} 
We will use the aforementioned lemmas to derive the convergence of the model.
Given Assumption \ref{Assump:Smooth}, 
%and $\mathbf{w}_{t+1}=\mathbf{w}_t-\frac{\eta_t}{K}\sum_{i\in \mathcal{S}_{t}}\widetilde{\mathbf{U}}^i_t$
we have
\begin{equation}
\label{EQ:LSmooth}
    \begin{aligned}
    &\mathbb{E}[F(\mathbf{w}_{t+1})]=F(\mathbf{w}_t) +\frac{L}{2}\mathbb{E}\Big\|\frac{\eta_t}{K}\sum_{i\in \mathcal{K}_{t}}\widetilde{\mathbf{U}}^i_t\Big\|^2+\nabla F(\mathbf{w}_t)^T\\
	&\qquad\mathbb{E}\Big[\frac{-\eta_t\sum_{i\in \mathcal{K}_{t}}\widetilde{\mathbf{U}}^i_t}{K}+\eta_tE\nabla F(\mathbf{w}_t)-\eta_tE\nabla F(\mathbf{w}_t)\Big] \\
	&\qquad=F(\mathbf{w}_t)-\eta_tE\|\nabla F(\mathbf{w}_t)\|^2+ \frac{L}{2}\mathbb{E}\Big\|\frac{\eta_t}{K}\sum_{i\in \mathcal{K}_{t}}\widetilde{\mathbf{U}}^i_t\Big\|^2\\
	&\qquad+\nabla F(\mathbf{w}_t)^T\mathbb{E}\Big[\frac{-\eta_t\sum_{i\in \mathcal{K}_{t}}\widetilde{\mathbf{U}}^i_t}{K}+\eta_tE\nabla F(\mathbf{w}_t)\Big].
    \end{aligned}
\end{equation}
% \begin{eqnarray}
%     \label{EQ:LSmooth}
% 	&&\mathbb{E}[F(\mathbf{w}_{t+1})]\notag\\
% 	&&=F(\mathbf{w}_t) +\frac{L}{2}\mathbb{E}\Big\|\frac{\eta_t}{K}\sum_{i\in \mathcal{K}_{t}}\widetilde{\mathbf{U}}^i_t\Big\|^2+\nabla F(\mathbf{w}_t)^T\notag\\
% 	&&\mathbb{E}\Big[\frac{-\eta_t\sum_{i\in \mathcal{K}_{t}}\widetilde{\mathbf{U}}^i_t}{K}+\eta_tE\nabla F(\mathbf{w}_t)-\eta_tE\nabla F(\mathbf{w}_t)\Big] \notag\\
% 	&&=F(\mathbf{w}_t)-\eta_tE\|\nabla F(\mathbf{w}_t)\|^2+ \frac{L}{2}\mathbb{E}\Big\|\frac{\eta_t}{K}\sum_{i\in \mathcal{K}_{t}}\widetilde{\mathbf{U}}^i_t\Big\|^2\notag\\
% 	&&+\nabla F(\mathbf{w}_t)^T\mathbb{E}\Big[\frac{-\eta_t\sum_{i\in \mathcal{K}_{t}}\widetilde{\mathbf{U}}^i_t}{K}+\eta_tE\nabla F(\mathbf{w}_t)\Big].
% \end{eqnarray}

%Considering that sampling and compression are unbiased, 
The last term in the above formula can be evaluated as:
\begin{eqnarray}
    \label{EQ:InerProduct}
	&&\nabla F(\mathbf{w}_t)^T\mathbb{E}\Big[\frac{-\eta_t\sum_{i\in \mathcal{K}_{t}}\widetilde{\mathbf{U}}^i_t}{K}+\eta_tE\nabla F(\mathbf{w}_t)\Big]\notag\\
	&&=\nabla F(\mathbf{w}_t)^T\mathbb{E}\Big[-\eta_t\sum_{i=1}^Np_i\mathbf{U}^i_t+\eta_tE\nabla F(\mathbf{w}_t)\Big]\notag\\
	%&&=\nabla F(\mathbf{w}_t)^T\mathbb{E}\left[-\sum_{i=1}^Np_i\eta_t\sum_{j=0}^{E-1}\nabla F_i(\mathbf{w}^i_{t,j},\mathcal{B}^i_{t,j})+\eta_tE\nabla F(\mathbf{w}_t)\right]\notag\\
	%&&=\nabla F(\mathbf{w}_t)^T\notag\\
	%&&\quad\mathbb{E}\Big[-\sum_{i=1}^Np_i\eta_t\sum_{j=0}^{E-1}\nabla F_i(\mathbf{w}^i_{t,j})+\eta_tE\sum_{i=1}^Np_i\nabla %F_i(\mathbf{w}_t)\Big]\notag\\
	&&=\sqrt{\eta_tE}\nabla F(\mathbf{w}_t)^T\notag\\
	&&\quad-\frac{\sqrt{\eta_t}}{\sqrt{E}}\mathbb{E}\Big[\sum_{i=1}^Np_i\sum_{j=0}^{E-1}\left(\nabla F_i(\mathbf{w}^i_{t,j})-\nabla F_i(\mathbf{w}_t)\right)\Big]\notag\\
	&&\overset{(1)}{=}\frac{\eta_tE}{2}\|\nabla F(\mathbf{w}_t)\|^2-\frac{\eta_t}{2E}\mathbb{E}\Big\|\sum_{i=1}^Np_i\sum_{j=0}^{E-1}\nabla F_i(\mathbf{w}^i_{t,j})\Big\|^2\notag\\
	&&\quad+\frac{\eta_t}{2E}\mathbb{E}\Big\|\sum_{i=1}^Np_i\sum_{j=0}^{E-1}\left(\nabla F_i(\mathbf{w}^i_{t,j})-\nabla F_i(\mathbf{w}_t)\right)\Big\|^2\notag\\
	%&&\le\frac{\eta_tE}{2}\|\nabla F(\mathbf{w}_t)\|^2+\frac{\eta_tN}{2}\sum_{i=1}^Np_i^2\sum_{j=0}^{E-1}\mathbb{E}\|(\nabla F_i(\mathbf{w}^i_{t,j})-\nabla F_i(\mathbf{w}_t))\|^2-\frac{\eta_t}{2E}\mathbb{E}\|\sum_{i=1}^Np_i\sum_{j=0}^{E-1}\nabla F_i(\mathbf{w}^i_{t,j})\|^2\notag\\
	&&\le\frac{\eta_tE}{2}\|\nabla F(\mathbf{w}_t)\|^2-\frac{\eta_t}{2E}\mathbb{E}\Big\|\sum_{i=1}^Np_i\sum_{j=0}^{E-1}\nabla F_i(\mathbf{w}^i_{t,j})\Big\|^2\notag\\
	&&\quad+\frac{\eta_tNL^2}{2}\sum_{i=1}^Np_i^2\sum_{j=0}^{E-1}\mathbb{E}\left\| \mathbf{w}^i_{t,j}-\mathbf{w}_t\right\|^2,
	%&&\quad-\frac{\eta_t}{2E}\mathbb{E}\left\|\sum_{i=1}^Np_i\sum_{j=0}^{E-1}\nabla F_i(\mathbf{w}^i_{t,j})\right\|^2,
	%&&\le\frac{\eta_tE}{2}\|\nabla f(\mathbf{w}_t)\|^2+\frac{\eta_tNEL^2}{2}\sum_{i=1}^Np_i^2\left(5E\eta_t^2(\sigma_L^2+6E\sigma_G^2)+30E^2\eta_t^2\|\nabla f(\mathbf{w}_t)\|^2 \right) \notag\\
	%&&\qquad-\frac{\eta_t}{2E}\mathbb{E}\|\sum_{i=1}^Np_i\sum_{j=0}^{E-1}\nabla F_i(\mathbf{w}^i_{t,j})\|^2\notag\\
	%&&\le\eta_tE\left(\frac{1}{2}+15NE^2\eta_t^2L^2\sum_{i=1}^Np_i^2 \right) \|\nabla f(\mathbf{w}_t)\|^2+ \frac{5NE^2\eta_t^3L^2\sum_{i=1}^Np_i^2}{2}(\sigma_L^2+6E\sigma_G^2)\notag\\
	%&&\qquad-\frac{\eta_t}{2E}\mathbb{E}\|\sum_{i=1}^Np_i\sum_{j=0}^{E-1}\nabla F_i(\mathbf{w}^i_{t,j})\|^2\notag\\
\end{eqnarray}
where $(1)$ is derived from $\mathbf{x}^T\mathbf{y}=\frac{\|\mathbf{x}\|^2+\|\mathbf{y}\|^2-\|(\mathbf{x}-\mathbf{y})\|^2}{2}$. 

The third term in Eq.~\eqref{EQ:LSmooth} can be evaluated as:
\begin{eqnarray}
\label{EQ:CompressionError}
	&&\mathbb{E}\big\|\frac{\eta_t}{K}\sum_{i\in \mathcal{K}_{t}}\widetilde{\mathbf{U}}^i_t\big\|^2=\frac{\eta_t^2}{K^2}\mathbb{E}\big\|\sum_{i\in \mathcal{K}_{t}}\widetilde{\mathbf{U}}^i_t\big\|^2\notag\\
	&&=\frac{\eta_t^2}{K^2}\mathbb{E}\big(\big\|\sum_{i\in \mathcal{K}_{t}}\big(\widetilde{\mathbf{U}}^i_t-\mathbf{U}^i_t\big)\big\|^2 + \big\|\sum_{i\in \mathcal{K}_{t}} \sum_{j=0}^{E-1}\nabla F_i(\mathbf{w}^i_{t,j})\big\|^2\big) \notag\\
	&&\quad+\frac{\eta_t^2}{K^2}\mathbb{E}\Big\|\sum_{i\in \mathcal{K}_{t}}\sum_{j=0}^{E-1}\left(\nabla F_i(\mathbf{w}^i_{t,j},\mathcal{B}^i_{t,j})-\nabla F_i(\mathbf{w}^i_{t,j})\right)\Big\|^2\notag\\
	&&\le \eta_t^2(\sum_{i=1}^Np_iJ^i_t+E^2\sigma^2)+\frac{\eta_t^2}{K^2}\mathbb{E}\Big\|\sum_{i\in \mathcal{K}_{t}} \sum_{j=0}^{E-1}\nabla F_i(\mathbf{w}^i_{t,j})\Big\|^2 \notag\\
	&&\le \eta_t^2\sum_{i=1}^Np_iJ^i_t+\frac{\eta_t^2}{K}\sum_{i=1}^Np_i\mathbb{E}\Big\|\sum_{j=0}^{E-1}\nabla F_i(\mathbf{w}^i_{t,j})\Big\|^2\\
	&&\quad+E^2\eta_t^2\sigma^2+\frac{\eta_t^2(K-1)}{K}\mathbb{E}\Big\|\sum_{i=1}^Np_i\sum_{j=0}^{E-1}\nabla F_i(\mathbf{w}^i_{t,j})\Big\|^2.\notag
\end{eqnarray}

We next evaluate the second term in the above formula.
\begin{eqnarray}
    \label{EQ:NablaFi}
    &&\mathbb{E}\Big\|\sum_{j=0}^{E-1}\nabla F_i(\mathbf{w}^i_{t,j})\Big\|^2=\mathbb{E}\Big\|\sum_{j=0}^{E-1}(\nabla F_i(\mathbf{w}^i_{t,j})-\nabla F_i(\mathbf{w}^i_{t})\notag\\
    &&\qquad+\nabla F_i(\mathbf{w}^i_{t})-\nabla F(\mathbf{w}^i_{t})+\nabla F(\mathbf{w}^i_{t}) ) \Big\|^2\\
    &&\le3E\Big(L^2\sum_{j=0}^{E-1}\mathbb{E}\|\mathbf{w}^i_{t,j}-\mathbf{w}^i_{t}\|^2+E\Gamma_n+E\|\nabla F(\mathbf{w}_t)\|^2\Big).\notag
    %&&\le15E^3L^2\eta_t^2(\sigma_L^2+6E\sigma_G^2)+(90E^4L^2\eta_t^2+3E^2)\|\nabla f(\mathbf{w}_t)\|^2+3E^2\sigma_G^2
\end{eqnarray}

Substituting Lemma~\ref{Lemma:LocalBound}, Eqs.~\eqref{EQ:InerProduct}, \eqref{EQ:CompressionError}, \eqref{EQ:NablaFi} into Eq.~\eqref{EQ:LSmooth}, 
\begin{eqnarray}
	&&\mathbb{E}[F(\mathbf{w}_{t+1})]
% 	&&\le F(\mathbf{w}_t)-\eta_tE\|\nabla F(\mathbf{w}_t)\|^2+ \frac{L}{2}\mathbb{E}\left\|\frac{\eta_t}{K}\sum_{i\in \mathcal{K}_{t}}\widetilde{\mathbf{U}}^i_t\right\|^2\notag\\
% 	&&+\nabla F(\mathbf{w}_t)^T\mathbb{E}\left[\frac{-\eta_t\sum_{i\in \mathcal{K}_{t}}\widetilde{\mathbf{U}}^i_t}{K}+\eta_tE\nabla F(\mathbf{w}_t)\right]\notag\\
	\le F(\mathbf{w}_t)-\eta_tE\|\nabla F(\mathbf{w}_t)\|^2\notag\\
	&&\quad+ \eta_tE\Big(\frac{1}{2}+15NE^2\eta_t^2L^2\sum_{i=1}^Np_i^2 \Big) \|\nabla F(\mathbf{w}_t)\|^2\notag\\
	&&\quad+ \frac{5NE^2\eta_t^3L^2\sum_{i=1}^Np_i^2}{2}(\sigma^2+6E\Gamma_n)\notag\\
	&&\quad-\frac{\eta_t}{2E}\mathbb{E}\Big\|\sum_{i=1}^Np_i\sum_{j=0}^{E-1}\nabla F_i(\mathbf{w}^i_{t,j})\Big\|^2\notag\\
	&&\quad+ \frac{L}{2}\Big(\eta_t^2\sum_{i=1}^Np_iJ^i_t+\frac{\eta_t^2}{K}\sum_{i=1}^Np_i\mathbb{E}\Big\|\sum_{j=0}^{E-1}\nabla F_i(\mathbf{w}^i_{t,j})\Big\|^2\notag\\
	&&\quad+E^2\eta_t^2\sigma^2+\frac{\eta_t^2(K-1)}{K}\mathbb{E}\Big\|\sum_{i=1}^Np_i\sum_{j=0}^{E-1}\nabla F_i(\mathbf{w}^i_{t,j})\Big\|^2 \Big) \notag\\
% 	&&\overset{a}{\le} F(\mathbf{w}_t)-\eta_tE(\frac{1}{2}-15NE^2\eta_t^2L^2\sum_{i=1}^Np_i^2)\|\nabla F(\mathbf{w}_t)\|^2\notag\\
% 	&&\qquad+\frac{5NE^2\eta_t^3L^2\sum_{i=1}^Np_i^2}{2}(\sigma_L^2+6E\Gamma_n)\notag\\
% 	&&\qquad+\frac{L\eta_t^2}{2}\sum_{i=1}^Np_iJ^i_t+\frac{LE^2\eta_t^2\sigma_L^2}{2}\notag\\
% 	&&\qquad+\frac{L\eta_t^2}{2K}\sum_{i=1}^Np_i\mathbb{E}\left\|\sum_{j=0}^{E-1}\nabla F_i(\mathbf{w}^i_{t,j})\right\|^2\notag\\
	&&\le F(\mathbf{w}_t)-\eta_tE\|\nabla f(\mathbf{w}_t)\|^2\notag\\
	&&\Big(\frac{1}{2}-15NE^2\eta_t^2L^2\sum_{i=1}^Np_i^2-\frac{L\eta_t}{2K}\left(90E^3L^2\eta_t^2+3E \right) \Big)\notag\\
	&&+\Big[\frac{5NE^2\eta_t^3L^2\sum_{i=1}^Np_i^2}{2}+\frac{15E^3L^3\eta_t^4}{2K}\Big]  (\sigma^2+6E\Gamma_n)\notag\\
	&&+\frac{L\eta_t^2}{2}\sum_{i=1}^Np_iJ^i_t+\frac{LE^2\eta_t^2\sigma^2}{2}+\frac{3E^2L\eta_t^2\Gamma_n}{2K}\notag\\
	&&+\Big(\frac{L\eta_t^2(K-1)}{2K}-\frac{\eta_t}{2E}\Big)\mathbb{E}\Big\|\sum_{i=1}^Np_i\sum_{j=0}^{E-1}\nabla F_i(\mathbf{w}^i_{t,j})\Big\|^2\notag\\
	&&\overset{(1)}{\le} F(\mathbf{w}_t)-c\eta_tE\|\nabla f(\mathbf{w}_t)\|^2\notag\\
	&&+\Big[\frac{5NE^2\eta_t^3L^2\sum_{i=1}^Np_i^2}{2}+\frac{15E^3L^3\eta_t^4}{2K}\Big]  (\sigma^2+6E\Gamma_n)\notag\\
	&&+\frac{L\eta_t^2}{2}\sum_{i=1}^Np_iJ^i_t+\frac{LE^2\eta_t^2\sigma^2}{2}+\frac{3E^2L\eta_t^2\Gamma_n}{2K},
%		&&\le f(\mathbf{w}_t)-\eta_tE(\frac{1}{2}-15E^2\eta_L^2L^2)\|\nabla f(\mathbf{w}_t)\|^2\notag\\
%		&&\qquad+\frac{5E^2\eta_t^2L^2}{2}(\sigma_L^2+6E\sigma_G^2)-\frac{\eta_t}{2E}\sum_{i=0}^N\mathbb{E}\|\sum_{j=0}^{E-1}\nabla F_i(\mathbf{w}^i_{t,j})\|^2   \notag\\
%		&&\qquad+ \frac{L}{2}\eta^2\left(\frac{1}{N}\sum_{i=0}^NJ^i_t+\eta_L^2E^2G^2 \right) \notag\\
%		&&\le f(\mathbf{w}_t)-\eta\eta_LE(\frac{1}{2}-15E^2\eta_L^2L^2)\|\nabla f(\mathbf{w}_t)\|^2\notag\\
%		&&\qquad+\frac{5E^2\eta\eta_L^2L^2}{2}(\sigma_L^2+6E\sigma_G^2)+ \frac{L}{2}\eta^2\left(\frac{1}{N}\sum_{i=0}^NJ^i_t+\eta_L^2E^2G^2 \right) \notag\\
%		&&\qquad -\frac{\eta\eta_L}{2EN}\left(15NE^3L^2\eta_L^2(\sigma_L^2+6E\sigma_G^2)+(90NE^4L^2\eta_L^2+3NE^2)\|\nabla f(\mathbf{w}_t)\|^2+3NE^2\sigma_G^2 \right) \notag\\
%		&&\le f(\mathbf{w}_t)-\eta\eta_LE\left( \frac{1}{2}-15E^2\eta_L^2L^2-\frac{1}{2N}(90NE^2L^2\eta_L^2+3N)\right) \|\nabla f(\mathbf{w}_t)\|^2\notag\\
%		&&\qquad+\frac{5E^2\eta\eta_L^2L^2}{2}(\sigma_L^2+6E\sigma_G^2)+ \frac{L}{2}\eta^2\left(\frac{1}{N}\sum_{i=0}^NJ^i_t+\eta_L^2E^2G^2 \right) \notag\\
%		&&\qquad -\frac{\eta\eta_L}{2EN}\left(15NE^3L^2\eta_L^2(\sigma_L^2+6E\sigma_G^2)++3NE^2\sigma_G^2 \right) \notag\\
%		&&\le f(\mathbf{w}_t)-\eta\eta_LEc \|\nabla f(\mathbf{w}_t)\|^2\notag\\
%		&&\qquad+\frac{5E^2\eta\eta_L^2L^2}{2}(\sigma_L^2+6E\sigma_G^2)+ \frac{L}{2}\eta^2\left(\frac{1}{N}\sum_{i=0}^NJ^i_t+\eta_L^2E^2G^2 \right) \notag\\
%		&&\qquad -\frac{\eta\eta_L}{2EN}\left(15NE^3L^2\eta_L^2(\sigma_L^2+6E\sigma_G^2)++3NE^2\sigma_G^2 \right) \notag\\
\end{eqnarray}
where $(1)$ follows from $\frac{L\eta_t^2(K-1)}{2K}-\frac{\eta_t}{2E}\le 0$ if $\eta_tEL\le \frac{K}{K-1}$
%, $b$ follows Eq.~\eqref{EQ:NablaFi} and $c$ follows
and $\frac{1}{2}-15NE^2\eta_t^2L^2\sum_{i=1}^Np_i^2-\frac{L\eta_t}{2K}\left(90E^3L^2\eta_t^2+3E \right)>c>0$ if $30NE^2\eta_t^2L^2\sum_{i=1}^Np_i^2+\frac{L\eta_t}{K}\left(90E^3L^2\eta_t^2+3E \right) < 1$.
% Therefore, the final result we can get is
% \begin{eqnarray}
% 	&&\mathbb{E}[F(\mathbf{w}_{t+1})]\notag\\
% 	&&\le F(\mathbf{w}_t)-c\eta_tE\|\nabla f(\mathbf{w}_t)\|^2\notag\\
% 	&&+\left[\frac{5NE^2\eta_t^3L^2\sum_{i=1}^Np_i^2}{2}+\frac{15E^3L^3\eta_t^4}{2K}\right]  (\sigma^2+6E\Gamma_n)\notag\\
% 	&&+\frac{L\eta_t^2}{2}\sum_{i=1}^Np_iJ^i_t+\frac{LE^2\eta_t^2\sigma^2}{2}+\frac{3E^2L\eta_t^2\Gamma_n}{2K}.
% \end{eqnarray}

 Theorem 2 is then proved. 
% The final convergence result obtained by transforming the above formula is
% \begin{eqnarray}
% 	&&\min_{t\in[T]}\mathbb{E}\|\nabla F(\mathbf{w}_t)\|^2\notag\\
% 	&&\le \frac{F_0-F_*}{c\eta_TTE}+\frac{(\sigma_L^2+6E\Gamma_n)}{c\eta_TTE}\notag\\
% 	&&\left[\frac{5NE^2L^2\sum_{i=1}^Np_i^2\sum_{t=0}^T\eta_t^3}{2}+\frac{15E^3L^3\sum_{t=0}^T\eta_t^4}{2K}\right]  \notag\\
% 	&&+\frac{L}{2c\eta_TTE}\sum_{t=0}^T\eta_t^2\sum_{i=1}^Np_iJ^i_t\notag\\
% 	&&+\left(\frac{LE^2\sigma_L^2}{2}+\frac{3E^2L\Gamma_n}{2K} \right)\frac{\sum_{t=0}^T\eta_t^2}{c\eta_TTE}
% \end{eqnarray}

\clearpage

\bibliographystyle{IEEEtran} 
\bibliography{reference}

\end{document}